%% file: iclr2026_conference.tex
\documentclass{article} 

\usepackage{iclr2026_conference,times}
\iclrfinalcopy
\input{math_commands.tex}

\usepackage{xcolor}
\definecolor{darkgrayish}{gray}{0.25} 
\definecolor{darkteal}{rgb}{0.0,0.25,0.25}

\usepackage[colorlinks,linkcolor=purple,citecolor=darkteal,urlcolor=teal]{hyperref}
\usepackage[utf8]{inputenc} 
\usepackage[T1]{fontenc}    
\usepackage{hyperref}       
\usepackage{url}            
\usepackage{booktabs}       
\usepackage{amsfonts}       
\usepackage{nicefrac}       
\usepackage{microtype}      
\usepackage[dvipsnames]{xcolor}
\usepackage{multirow,bm}
\usepackage{graphicx}
\usepackage{pifont}
\usepackage{amsmath}
\usepackage{amsthm}
\usepackage{algpseudocode}
\usepackage{algorithm}
\usepackage{enumitem}
\usepackage{physics,amssymb}
\usepackage{subcaption} 
\usepackage{colortbl} 

\newcommand{\ie}{\textit{i.e., }}
\newcommand{\eg}{\textit{e.g., }}

\newcommand{\etc}{\textit{etc.}}

\newcommand{\cf}{\textit{cf. }}
\newcommand{\aka}{\textit{aka. }}
\newcommand{\Prob}{\mathbb P}
\newcommand{\bR}{\mathbb R}
\newcommand{\replace}[2]{#2}
\newcommand{\xhdr}[1]{{\noindent\bfseries #1}.}
\newcommand{\xhdrnodot}[1]{{\noindent\bfseries #1}}
\newcommand{\revision}[1]{\textcolor{black}{#1}}

\newcommand{\rebuttal}[1]{\textcolor{black}{#1}}

\usepackage[capitalize,noabbrev]{cleveref}

\newcommand{\std}[1]{$\pm$\small{#1}}
\newcommand{\avgrankingacc}{77.96}
\newcommand{\meanimprov}{11.92}
\newcommand{\optimas}{\textsc{Optimas}}
\newcommand{\hotpot}{\textsc{HotpotQA}}
\newcommand{\pubmed}{\textsc{PubMedQA}}
\newcommand{\amazon}{\textsc{Amazon}}
\newcommand{\code}{\textsc{BigCodeBench}}
\newcommand{\stark}{\textsc{STaRK-Prime} }

\newcommand{\config}{configuration}

\newtheorem{thm}{Theorem}[section]

\newtheorem{lem}{Lemma}[section]

\newtheorem{asm}{Assumption}[section]

\newcommand{\iclr}[1]{\textcolor{black}{#1}}

\usepackage{tikz}

\newcommand{\obj}{component}
\newcommand{\Obj}{Component}
\newcommand{\RF}{Local Reward Function}
\newcommand{\abbr}{LRF}
\newcommand{\loc}{local reward}
\newcommand{\glob}{global reward}

\title{{\huge{\optimas}}: Optimizing Compound AI Systems with Globally Aligned Local Rewards}

\author{%
  Shirley Wu$^{*1}$,
  Parth Sarthi$^{*1}$, 
   Shiyu Zhao$^{*1}$, 
   Aaron Lee$^{1}$, 
  Herumb Shandilya$^{1}$, \\\textbf{Adrian Mladenic Grobelnik}$^{3}$, \textbf{Nurendra Choudhary}$^{2}$, 
  \textbf{ Eddie Huang}$^{2}$, \textbf{Karthik Subbian}$^{2}$, \\\textbf{Linjun Zhang}$^{4}$,
  \textbf{Diyi Yang}$^{1}$,
  \textbf{James Zou}$^{**1}$, \textbf{Jure Leskovec}$^{**1}$\\
    \thanks{Equal contribution; $^{**}$Equal supervision. Correspondence: {\small\texttt{\small<\{shirwu,psarthi\}@cs.stanford.edu>}}.}$^{1}$Stanford University\ \ \ $^{2}$Amazon\ \ \ $^{3}$Jožef Stefan Institute\ \ \  $^{4}$Rutgers University
     \\
  [0.1cm]{\quad\quad\quad\quad\quad\quad\quad\quad\quad\quad\textcolor{cyan}{\url{https://optimas.stanford.edu/}}}
}

\begin{document}

\maketitle

\input{chapters/1.abstract}

\input{chapters/2.introduction}

\input{chapters/3.related_work}

\input{chapters/4.1.formulation}

\input{chapters/4.2.intuitions}

\input{chapters/4.3.reward_func}

\input{chapters/4.5.adaptation}

\input{chapters/4.4.optimization}

\input{chapters/4.5.theory}

\input{chapters/5.1.experiment_setup}

\input{chapters/5.2.experiment_main}

\input{chapters/7.conclusion}

\bibliography{iclr2026_conference}
\bibliographystyle{iclr2026_conference}

\appendix

\input{chapters/9.appendix}

\end{document}

%% file: math_commands.tex
\usepackage{amsmath,amsfonts,bm}

\def\eqref#1{equation~\ref{#1}}

\def\1{\bm{1}}

\DeclareMathAlphabet{\mathsfit}{\encodingdefault}{\sfdefault}{m}{sl}
\SetMathAlphabet{\mathsfit}{bold}{\encodingdefault}{\sfdefault}{bx}{n}

\newcommand{\E}{\mathbb{E}}



%% file: chapters/1.abstract.tex
\vspace{-5pt}
\begin{abstract}
    
    Compound AI systems integrating multiple \obj{}s, such as Large Language Models, specialized tools, and traditional machine learning models, are increasingly deployed to solve complex real-world tasks. However, optimizing compound systems remains challenging due to their non-differentiable structures and diverse \config{} types across \obj{}s, including prompts, hyperparameters, and model parameters.
    To address this challenge, we propose \optimas{}, a unified framework for effective optimization of compound systems. The core idea of \optimas{} is to maintain one \textit{\RF{}} (\abbr{}) per \obj{}, each satisfying a \textit{local–global alignment} property, \ie each \obj{}'s \loc{} correlates with the global system performance. In each iteration, \optimas{} efficiently adapts the \abbr{}s to maintain this property while simultaneously maximizing each \obj{}'s \loc{}. This approach enables independent updates of heterogeneous configurations using the designated optimization method, while ensuring that local improvements consistently lead to performance gains. 
    We present extensive evaluations across five real-world compound systems to demonstrate that \optimas{} outperforms strong baselines 
    by an average improvement of \meanimprov\%, offering a general and effective approach for improving compound systems.
\end{abstract}

%% file: chapters/2.introduction.tex
\begin{figure}[h]
    \centering
    \includegraphics[width=1.\textwidth]{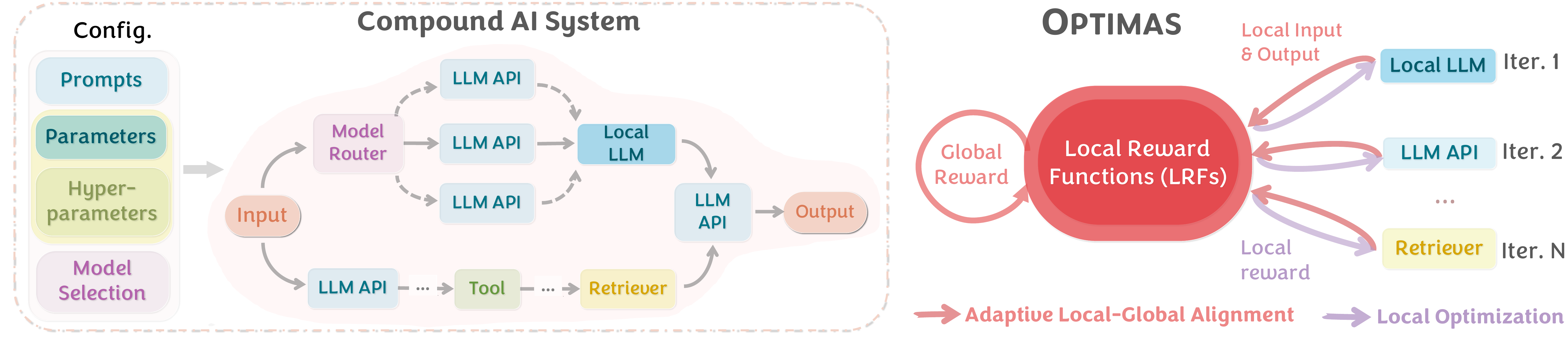}
    \vspace{-15pt}
    \caption{Overview. Given a compound AI system's heterogeneous configurations (\eg prompts, parameters) across multiple \obj{}s, \optimas{} maintains globally aligned \textit{\RF{}}s (\abbr{}s) as the system evolves, where each supervises a \obj{} and assigns higher \loc{}s to outputs with higher system performance (\aka \glob{}s). It iteratively adapts \abbr{}s and optimizes each \obj{} to maximize its \loc{} for effective system optimization.
    }
    \label{fig:overview}
    \vspace{-5pt}
\end{figure}

\begin{figure}[t]
    \centering
    \includegraphics[width=0.95\textwidth]{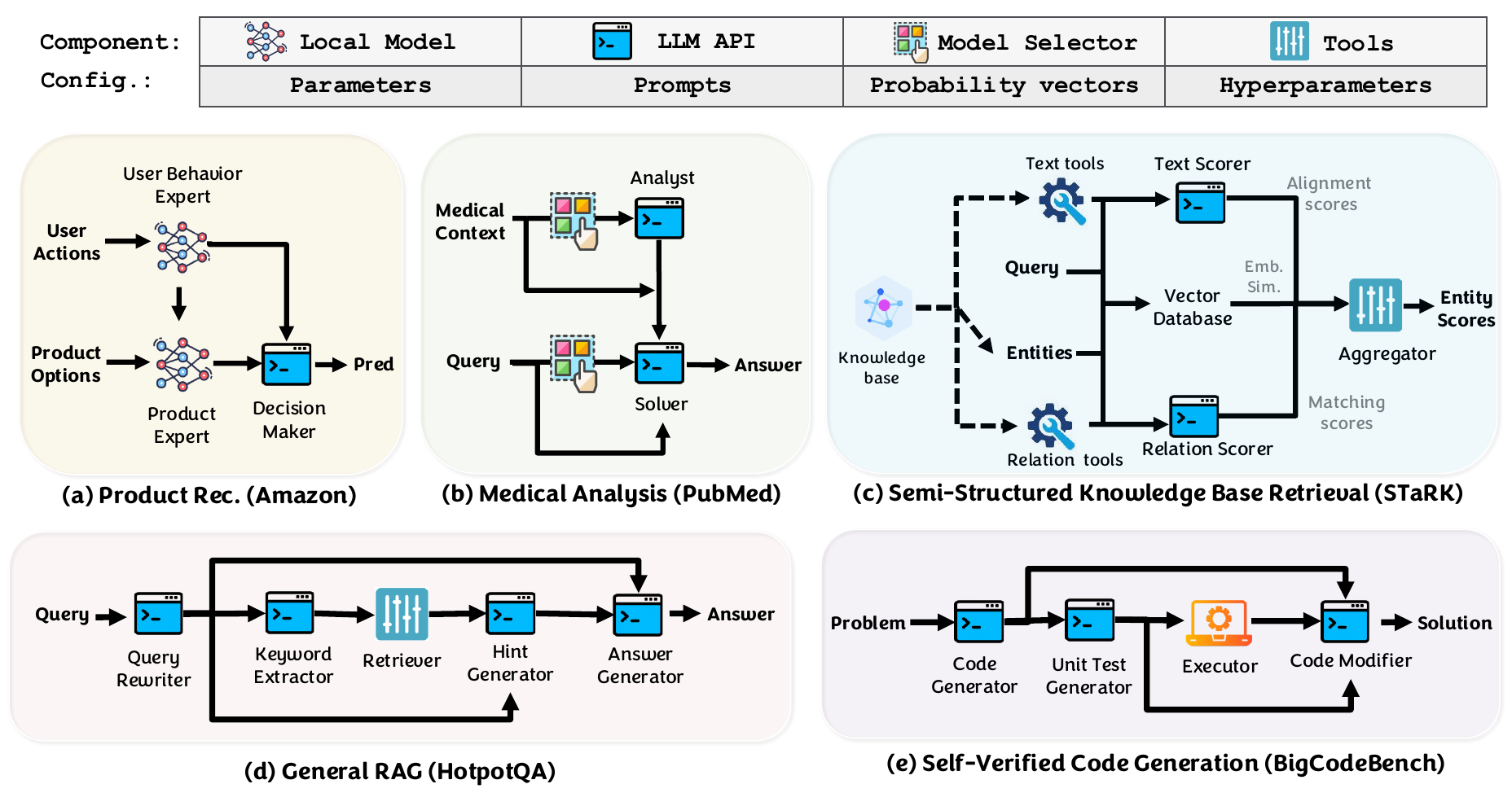}
    \vspace{-5pt}
    \caption{Five real-world and challenging compound AI systems. The goal is to automatically optimize the \config{} across a heterogeneous set of components and parameters, \eg model parameters, prompts, model selection choice, and hyperparameters. See Appendix~\ref{app:system} for details. 
    }
    \label{fig:tasks}
    \vspace{-12pt}
\end{figure}

\section{Introduction}
Modern AI systems increasingly employ compound systems that integrate multiple complex \obj{}s, such as Large Language Models (LLMs), tool/function calls, and traditional machine learning models like retrievers~\citep{textgrad,dsp,multiagent_debate}.
These \obj{}s collaborate to process heterogeneous data sources and solve complex tasks through specialized subtask allocation~\citep{compound-ai-blog,mars,scalable,orchestrate,model_selection}.
While compound AI systems have yielded performance advantages over monolithic models~\citep{multiagent_debate,agentcoder,alphacode2,rag,dylan}, they can be highly sensitive to the failure of individual \obj{}s, which leads to cascading failures in the final results~\citep{multiagent_fail, faulty_multiagent_resilience, reasoning_disruption, interm_failures}. 
For example, if an LLM misinterprets an input query, it can retrieve irrelevant or misleading information. This leads to subsequent tool calls operating on incorrect inputs, producing unreliable outputs throughout the system.
Therefore, optimizing these systems as a whole is crucial for maintaining reliability and global system performance (\ie \glob{}s).

However, optimizing these compound AI systems end-to-end is fundamentally challenging due to their non-differentiable nature. Second, it is hard to jointly optimize heterogeneous \config{}s (textual or numerical, continuous or discrete) from individual \obj{}s, including prompts, hyperparameters, model selections, and even model weights. Moreover, running the entire compound AI system during optimization to achieve \glob{} is costly.

Previous works have largely focused on optimizing specific \config{}s in isolation, such as optimizing prompts through textual feedback~\citep{textgrad, dsp, llm_as_opt, self_refine, avatar} or model selection through iterative search~\citep{model_selection,frugalgpt}. Yet these approaches can fail to capture critical bottlenecks. For example, a perfectly optimized prompt can struggle to compensate for a poorly chosen model. \iclr{Even when the \obj{}s are individually well optimized, they may still collaborate suboptimally, as the upstream \obj{} might not have visibility into which inputs are effective for the downstream \obj{}s. Consequently, previous methods typically require costly system runs over many configurations to identify the best configuration for the \obj{}s to work well together, leading to low data efficiency due to the large number of configurations.}


\xhdr{Present work}  
Here we propose \optimas{} (Figure~\ref{fig:overview}), a unified framework for the effective and data-efficient optimization of compound AI systems. The core idea is to \textit{learn} a \textbf{globally aligned} \textit{\RF{}} (\abbr{}) per \obj{}, such that independently maximizing a component’s \loc{} still reliably improves the \glob{}s. \revision{We show that under mild conditions, our approach converges reliably, providing strong theoretical guarantees.}
Furthermore,  since the learned \abbr{}s can be used to optimize \obj{}s locally, \optimas{} has \revision{higher data efficiency by avoiding extensive runs of the entire compound AI system to achieve high \glob{}.}  

Specifically, each \abbr{} estimates the contribution of a \obj{}’s output to the \glob{}. All \abbr{}s are implemented using a shared LLM backbone, with separate projection heads for each \obj{} to produce component-specific rewards. 
 We propose a \textbf{lightweight adaptation} mechanism using mini-batch preference data to ensure the \abbr{}s remain aligned with the evolving system configuration (Figure~\ref{fig:opt}). Leveraging this decentralized structure, \optimas{} applies specific optimization method to each \obj{} based on its configuration type. For example, reinforcement learning for model parameters~\citep{dpo,ppo} or metric-guided search for prompts and hyperparameters~\citep{llm_as_opt,mipro,search}. 
Overall, \optimas{} iteratively updates heterogeneous configurations towards a higher \glob{} by using each adaptive \abbr{} as an objective. \iclr{By optimizing a \obj{} to maximize its \loc{}, \optimas{} reduces the entire system runs to maintain higher data efficiency.}

We conduct extensive experiments to evaluate \optimas{} across five real-world compound systems (Figure~\ref{fig:tasks}), including challenging settings such as behavior-driven product recommendation and medical analysis. \optimas{} consistently outperforms strong baselines, achieving an average \revision{relative} improvement of \meanimprov\% \revision{with higher data efficiency}, while baseline methods occasionally improve performance. For example, while DSPy improves the performance on the multi-hop QA system, it may degrade performance on other tasks, such as product recommendation. In contrast, \optimas{} is the \emph{only} method that improves performance across \emph{all} five tasks, consistent with our theoretical guarantee (Section~\ref{theory}) that aligning local and global rewards enables effective optimization. 


%% file: chapters/3.related_work.tex
\begin{table}[t]
  \centering
  \caption{A comparison of \optimas{} with selected methods. \optimas{} optimizes compound systems with heterogeneous configurations and enables higher data efficiency with local optimization to reduce number of system runs. We prove \optimas{}'s convergence under mild conditions.
  }
  \label{tab:setting}
  \vspace{-8pt}
  \resizebox{1\textwidth}{!}{%
  \begin{tabular}{lcccc}
    \toprule
     & {Supports} & { Optimizes} 
     & Data
     & {Convergence}  \\
    & {compound system} 
    & {heterogeneous config.} 
    & efficiency
    & guarantee\\
    \midrule
    OPRO~\citep{llm_as_opt} & \textcolor{RedOrange}{\ding{55}} & \textcolor{RedOrange}{\ding{55}} & \textcolor{RedOrange}{\ding{55}} & \textcolor{RedOrange}{\ding{55}}\\
    
    DSPy~\citep{dsp} & \textcolor{LimeGreen}{\ding{52}} & \textcolor{RedOrange}{\ding{55}} & \textcolor{RedOrange}{\ding{55}} & \textcolor{RedOrange}{\ding{55}} \\
    TextGrad~\citep{textgrad} & \textcolor{LimeGreen}{\ding{52}} & \textcolor{RedOrange}{\ding{55}} & \textcolor{RedOrange}{\ding{55}} & \textcolor{RedOrange}{\ding{55}} \\
    LLMSelector~\citep{model_selection} & \textcolor{LimeGreen}{\ding{52}} & \textcolor{RedOrange}{\ding{55}} & \textcolor{RedOrange}{\ding{55}} & \textcolor{LimeGreen}{\ding{52}}\\
    \optimas{} & \textcolor{LimeGreen}{\ding{52}} 
    & \textcolor{LimeGreen}{\ding{52}} 
    & \textcolor{LimeGreen}{\ding{52}}
    & \textcolor{LimeGreen}{\ding{52}}\\
    \bottomrule
  \end{tabular}
  }
  \vspace{-10pt}
\end{table}

\section{Related Work}
    
\xhdr{Optimizing LLM single-step generation}
Prior work extensively optimizes prompts for Large Language Models (LLMs) in single-step generation to improve performance~\citep{llm_as_opt, self_refine, evoprompt, reflexion, admn, godel, chen2025bootstrapping, Williams92}, but these methods are limited in their ability to handle complex, multi-step tasks. For example, addressing complex queries often requires combining multiple \obj{}s—such as LLMs, tools, and machine learning predictors—to obtain accurate predictions.

\xhdr{Optimizing multi-\obj{}/multi-step generation}
Compound AI systems consisting of multiple \obj{}s enable more complex planning and specialized processing at each task step~\citep{dsp, react, dylan, multiagent_debate, aflow}. Previous studies typically optimize different \obj{}s separately, such as optimizing LLM prompts~\citep{dsp, avatar, textgrad}, fine-tuning model weights using supervised learning~\citep{sirius, a-bad-bo} or reinforcement learning~\citep{dcir, mmoa_rag}, developing model routing~\citep{model_selection} \revision{ and layer grouping~\citep{chen2023parameterefficientfinetuningdesignspaces}} strategies, and selecting hyperparameters~\citep{multiagent_ml, pmlr-v80-falkner18a, pmlr-v80-pham18a, liu2019dartsdifferentiablearchitecturesearch}.
In contrast, \optimas{} enables end-to-end optimization across \textit{all} \obj{}s. 

\xhdr{Reward modeling for multi-step tasks}
To provide more fine-grained supervision, recent works break down \glob{}s (\eg answer accuracy) into more targeted signals (\ie dense/process rewards). 
Representative approaches include leveraging or bootstrapping from human step-wise annotations~\citep{verify,abs-2211-14275,rprm};  hierarchical planning that assigns rewards to error correction steps~\citep{hrm}; using Monte Carlo Tree Search to assign credit to intermediate reasoning steps~\citep{math_shepherd,slrm,jiao2024lpr,chen-etal-2024-step,rewarding_pregress} or actions~\citep{armap,prm_agent}.
\iclr{Recently, \citet{a-bad-bo} leverage Bayesian optimization to decompose global losses into local losses for optimizing model weight.
}
\iclr{Optimas differs by dynamically aligning local rewards with global rewards through preference-based adaptation. This design is applicable to both differentiable and non-differentiable configurations, without requiring fixed decomposition or extensive retraining. The local optimization avoids extensive system runs and offers higher data efficiency. Moreover, we provide theoretical analysis to prove the convergence of our framework. We highlight our key contributions in Table~\ref{tab:setting}.
}

%

%% file: chapters/4.1.formulation.tex
\section{Problem Formulation: Optimizing Compound AI Systems}
\xhdr{Compound AI system}
A compound AI system is represented as a directed acyclic graph \(\mathcal{G} = (\mathcal{C}, \mathcal{E})\),  
where \(\mathcal{C} = \{ C_k \}_{k=1}^K\) is a set of \(K\) distinct \obj{}s (task nodes) and \(\mathcal{E}\) is the set of all possible directed edges between components.  
A \obj{} of the compound system can be an LLM, a general machine learning model, a model selector, \etc 
We denote the input and output to each component \(C_k\) as \(x_k\) and \(y_k\), respectively. 
The system input is treated as a source node $C_0$. 

The system can operate with dynamic planning: for each input instance \(x\), the connections \(\mathcal{E}(x) \subseteq \mathcal{E}\) between the components can be adaptive. 
A directed edge \((C_i, C_j) \in \mathcal{E}(x)\) indicates that the output of component \(C_i\) is routed as input to component \(C_j\) when processing instance \(x\). By default, we assume the \obj{} indices follow the topological order over \(\mathcal{E}\), meaning that $C_i$ is the upstream component of $C_j$ if $i<j$.

\xhdr{Component \config{}s} A \obj{} \(C_k:(x_k;\mathbf v_k)\mapsto y_k\) 
is controlled by a \config{} policy
\(\mathbf v_k\). 
The \config{} space \(\mathcal{V}\) can either be empty (indicating no optimizable \config{} for the component), discrete (\eg textual prompts or model selections), or continuous (\eg model parameters or hyperparameters). 
We denote the joint \config{} policy by
$
\mathbf v = (\mathbf v_1, \dots, \mathbf v_K).
$


\xhdr{Forward execution}
For a given input \(x\) and \config{} policy \(\mathbf v\),  
the system executes components in topological order over the edge set \(\mathcal{E}(x)\):
$
y_k = C_k\bigl(\{ y_i \mid C_i \in pa(C_k) \};\, \mathbf v_k \bigr), 
$
where $pa(C_k)$ denotes all the parents of component $C_k$ over \(\mathcal{E}(x)\). 
For clarity, we define the overall system as \( f(x;\, \mathbf v) := y \), where \(y\) is a collection of outputs from one or more components.  

\xhdr{Optimization objective}
Given a dataset $\mathcal{D}$ with initial inputs and a user-defined \glob{} function \(R: \mathcal{X} \times \mathcal{Y} \to \mathbb{R}\)  
that evaluates the final system output, 
the optimization goal is to find the \config{} policy \(\mathbf v^\star(x)\) that maximizes the expected \glob{}:
\begin{equation}
\mathbf v^\star = \arg\max_{\mathbf v} \; \mathbb{E}_{x\sim\mathcal{D}} \bigl[ R(x, f(x;\, \mathbf v)) \bigr].
\label{eq:objective}
\end{equation}
\vspace{-20pt}

%% file: chapters/4.2.intuitions.tex
\section{\optimas{}: Globally Aligned Local Rewards for Optimization}
\label{sec:method}

\xhdr{Challenges} Directly optimizing the objective in Eq.~\ref{eq:objective} is difficult. 
As the \config{} spaces are typically non-differentiable, gradient-based optimization cannot be used.
Moreover, each policy \(\mathbf v_k\) may control a different \config{} type, so the joint policy \(\mathbf v\) can span heterogeneous spaces. 
Therefore, previous efforts~\citep{textgrad,sirius,model_selection,dsp} largely focus on optimizing the policy for single types of \config{}s, which simplifies the optimization problem; however, this also leads to suboptimal compound systems. 

\xhdr{Key intuition}  
To address the challenges, our approach (Figure~\ref{fig:opt}) learns \RF{}s (\abbr{}s) that align with the \glob{} for individual \obj{}s, allowing \revision{local and independent optimization} on heterogeneous \obj{}s using different optimization approaches.

\iclr{Such \textbf{local-global alignments} (Section~\ref{sec:unit}) encourage that the \glob{} to increase during local optimizations (Section~\ref{sec:optimize}). }
\iclr{Moreover, as the system \config{}s change during optimization, the \abbr{}s should be adapted to remain aligned.}
To ensure alignment, \optimas{} employs a \textbf{lightweight adaptation} mechanism that updates \abbr{}s \revision{with minimal data sampled from the system}, preserving consistency with the \glob{} (Section~\ref{sec:evolve}).


%% file: chapters/4.3.reward_func.tex
\subsection{Learning \RF{}s}
\label{sec:unit}

\xhdr{Definition (\RF{} (\abbr))} An \abbr{}  on \obj{} $C_k$ is defined as $r_k: (x_k, y_k) \rightarrow \mathbb{R}$, which evaluates the component’s output \(y_k\) given the provided context \(x_k\).

\xhdr{Implementation} We implement all \abbr{}s with a LLM backbone $\phi$ and separate linear heads $h_k$ for a component $C_k$. 
The backbone encodes the concatenated text inputs 
$[x_k, y_k]$ into an embedding, and the corresponding head projects this embedding to a scalar reward value. Using such a multitask neural network ensures scalability with large number of components and reduces memory costs. Specifically, each \abbr{} is modeled as:
\begin{equation}
    r_k(x_k, y_k) = h_k \circ \phi([x_k, y_k])), 
\quad \text{for all } k \text{ if } \mathbf{v}_k \text{ is non-empty. } 
\end{equation}


\xhdr{Property (Local–global alignment)}
An \abbr{} \(r_k\) is said to be \emph{aligned} with the \glob{}
\(R\) if, for every input \(x\) and for any two candidate outputs
\(y_k^{+}, y_k^{-}\) of \(C_k\),
\begin{equation}
\label{eq:alignment}
\begin{aligned}
&r_k(x_k, y_k^{+}) \;\ge\; r_k(x_k, y_k^{-}) \\[2pt]
&\Longrightarrow\;
\mathbb{E}_{\text{downstream}}
\Bigl[ R\bigl(x, f(x;\mathbf v_{-k}\bigr) \,\big|\, y_k^{+}\Bigr]
\ge
\mathbb{E}_{\text{downstream}}
\Bigl[ R\bigl(x, f(x;\mathbf v_{-k})\bigr) \,\big|\, y_k^{-} \Bigr],
\end{aligned}
\end{equation}
where  
\(\mathbf v_{-k}\) denotes the configurations of all \emph{downstream}
components (those that directly or indirectly receive information originating
from \(C_k\)). 
The expected \glob{} for each candidate output is estimated via Monte Carlo sampling. This involves executing the downstream components with the candidate output \rebuttal{and the outputs from the non-downstream components} fixed, capturing their stochasticity, and averaging the resulting \glob{}s from the final system outputs.

\xhdr{Objective of reward functions}  
\iclr{To make each \abbr{} \(r_k\) conform to the local-global alignment property, we collect \(\mathcal{D}_k(\mathbf v)\), a preference dataset under the current
system configuration~\(\mathbf v\), and train each \(r_k\) using a pairwise
log‑sigmoid ranking loss}:
\begin{equation}
\label{eq:reward_model}
\mathcal{L}_k(\mathcal{D}_k(\mathbf v))
\,=\;
-\,
\mathbb{E}_{(x_k,\,y_k^{+},\,y_k^{-}) \,\sim\, \mathcal{D}_k(\mathbf v)}
\!\Bigl[
    \log \sigma\bigl(r_k(x_k, y_k^{+}) - r_k(x_k, y_k^{-})\bigr)
\Bigr],
\end{equation}

The collection of \(\mathcal{D}_k(\mathbf v)\) follows the following steps: 1) execute the compound system up to \(C_k\) and record the partial trajectory
\(\bigl\langle x,\,(x_1,y_1),\dots,(x_{k-1},y_{\,k-1})\bigr\rangle\);  
2) sample two candidate outputs for \(C_k\) (\eg via higher-temperature decoding or alternate hyperparameters); and 3) estimate their expected task metrics according to the expectation terms on the right-hand side of Eq.~\ref{eq:alignment}. The output with the higher expected value is labeled as \(y_k^{+}\), and the other as \(y_k^{-}\) in \(\mathcal{D}_k(\mathbf{v})\).

%% file: chapters/4.5.adaptation.tex
\begin{figure}[t]
    \centering
    \includegraphics[width=1\textwidth]{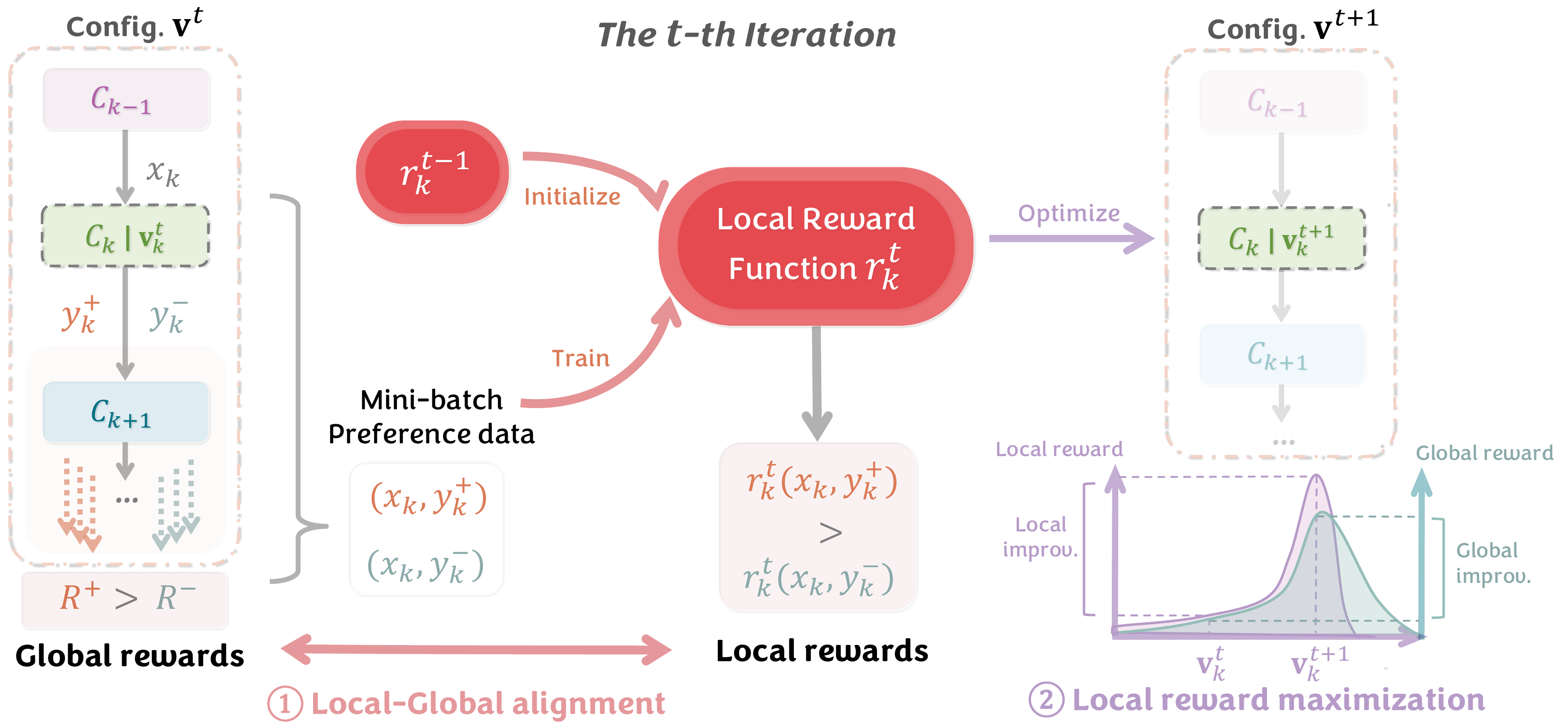}
    \vspace{-10pt}
    \caption{\optimas{} optimization iteration.
    At each iteration, \optimas{} updates a component \(C_k\) by first collecting a mini-batch of preference data and adapting its \RF{} \(r_k\) to remain aligned with the global task metric. This alignment \revision{helps ensure} that optimizing the component to maximize its \loc{} also improves the \glob{}.}
    \label{fig:opt}
    \vspace{-8pt}
\end{figure}

\subsection{Adaptive \RF{}s}
\label{sec:evolve}

\xhdr{Problem: Misaligned \abbr{}s in the evolving system}  
As the system configuration changes during optimization, \abbr{}s trained under a previous configuration \(\mathbf{v}^t\) may become inaccurate under the updated configuration \(\mathbf{v}^{t+1}\).  
\iclr{Specifically, (1) after updating \(C_k\), the same outputs from its upstream component \(C_i\) ($i<k$) may lead to different \glob{}, making $r_i$ misaligned.  
(2) Its downstream components \(C_j\) ($j>k$) now receive inputs generated by the updated \(C_k\), which may fall outside the distribution seen by their \abbr{}s.  }
These shifts accumulate over time, degrading the local–global alignment property (Eq.~\ref{eq:alignment}) that \abbr{}s are designed to satisfy.

However, retraining all \abbr{}s from scratch after every \config{} update \iclr{is expensive}.
To address this, we develop a lightweight adaptation strategy that incrementally refines the \abbr{}s as the system changes, maintaining alignment without full retraining. 


\xhdr{Stage 1: Initial reward modeling}  
Given the initial system configuration and a dataset with initial inputs, we first construct an offline preference dataset for each \obj{} and train its \abbr{} to convergence.  
This offline phase establishes well-aligned \abbr{}s that accurately reflect each component’s contribution to the \glob{}.

\xhdr{Stage 2: Online reward function adaptation}  
\iclr{
\rebuttal{When any configuration changes,} we sample a small batch of input data and construct a mini-batch of preference data $\mathcal{B}_k$ for each component \(C_k\) using the steps described in Section~\ref{sec:unit}. We then optimize the \abbr{} on $C_k$ on the objective
$
\label{eq:online_adaptation}
\mathcal{L}_k(\mathcal{B}_k)
$
following the definition in Eq.~\ref{eq:reward_model}.
To enable stable optimization and improve data efficiency, we maintain a buffer of previous generated preference data into $\mathcal{B}_k$ .
This adaptation helps maintain the local–global alignment property in Eq.~\ref{eq:alignment}. 
}


%% file: chapters/4.4.optimization.tex
\subsection{Optimization with Globally aligned \RF{}s}
\label{sec:optimize}


\xhdr{Local Optimization} 
\iclr{As each \obj{} has its own \abbr{}}, \optimas{} flexibly applies a specialized optimization method for each component. See details in Appendix~\ref{app:opt_details}. Concisely,  
\vspace{-5pt}
\begin{itemize}[leftmargin=*, itemsep=2pt]
    \item For textual prompts, we use prompt optimization algorithms such as OPRO~\citep{llm_as_opt} that ranks candidate prompts by their average \loc{} and select the best-performing one. 
    \item For components that are trainable models (\eg an LLM), we apply reinforcement learning—such as Proximal Policy Optimization (PPO)~\citep{ppo}—using the \abbr{} as the critic. 
    \item For discrete or low-dimensional continuous configurations, such as model selection or hyperparameter tuning, we construct a probability distribution over candidate values based on their \loc{}s and sample from it to update the configuration.
\end{itemize}

\xhdr{Overall algorithm (Figure~\ref{fig:opt})}  \iclr{Starting with the initial system configuration,
\optimas{} leverages initial reward modeling to learn a set of \abbr{}s that are well-aligned with the \glob{}. At each iteration optimization, \optimas{} randomly selects a component to optimize, conducts local optimization, \rebuttal{and if the local configuration change leads to improved global reward, it updates the system and adapts the \abbr{} using minimal amount of data}. To prevent potential cascading errors, the new configuration is accepted only if it improves the global reward on a small validation set.
Since the optimization with \abbr{}s is conducted locally, the number of system runs to achieve a high \glob{} is reduced, as we later show in the experiments.}
The detailed algorithm is provided in Appendix~\ref{app:algorithm}.




%% file: chapters/4.5.theory.tex
\subsection{Theoretical Insights}
\label{theory}
We prove that the local–global alignment property holds for the \abbr{}s constructed in Section~\ref{sec:unit}.

\begin{thm}[Informal]\label{thm:alignment}
Under regularity conditions, the \rebuttal{minimizer} of \eqref{eq:reward_model} satisfies the local-global alignment property (\eqref{eq:alignment}). In addition, maximizing $r_k\bigl(x_k,\,C_k(x_k;\mathbf{v}_k)\bigr)$ over $\mathbf{v}_k$ and maximizing $R(x,f(x;\mathbf{v}_{-k})\mid C_k(x_k;\mathbf{v}_k))$ over $\mathbf{v}_k$ will yield the same solution. 
\end{thm}
We defer the formal statement for this theorem to appendix \ref{app:theory}.
Since solving \eqref{eq:objective} is generally challenging, we introduce some regularity conditions to make the convergence analysis tractable. 

As the configurations $\mathbf{v}$ are heterogeneous, where some of the coordinates are discrete, and some are continuous, without loss of generality, we assume the first $M$ configurations $\mathbf{v}^{(1)}=\{\mathbf{v}_1,..., \mathbf{v}_M\}$ are continuous, and the last $(K-M)$ configurations $\mathbf{v}^{(2)}=\{\mathbf{v}_{M+1},...,\mathbf{v}_K\}$ are discrete. We write the objective function $\mathbb{E}_{x\sim\mathcal{D}} \bigl[ R(x, f(x;\, \mathbf v) \bigr]$ as $l(\mathbf{v})=l(\mathbf{v}_1,\mathbf{v}_2,...,\mathbf{v}_K):=l(\mathbf{v}^{(1)}, \mathbf{v}^{(2)})$.

\begin{asm}\label{asm: unique}
    Suppose for any given configuration $\mathbf{v}^{(2)}$,  the initial level set $\{\mathbf{v}^{(2)}: l(\mathbf{v}^{(1)}, \mathbf{v}^{(2)})\le l(\mathbf{v}^{0, (1)}, \mathbf{v}^{(2)})\}$ is a compact set, where $\mathbf{v}^{0, (1)}$ is the initialization used in the algorithm for $\mathbf{v}^{(1)}$. In addition, for every component $k$ and every fixed $\mathbf{v}_{-k}$, $l(\cdot, \mathbf{v}_{-k})$ has a unique maximizer.
\end{asm}

\begin{thm}\label{thm:convergence}
    Under Assumption~\ref{asm: unique}, the algorithm will converge to the component-wise maximum, that is, the limit point $\mathbf{v}^*$ satisfies $$
l(\mathbf{v}^*)\ge l(\mathbf{v}_k, \mathbf{v}^*_{-k}),
    $$
    for any $k\in[K]$ and any $\mathbf{v}_k$. 
\end{thm}
In fact, our theoretical analysis shows that by conducting local optimization, \optimas{} is essentially performing coordinate maximization. Therefore, existing convergence results for coordinate maximization directly apply. \rebuttal{Also, note that the block-coordinate (round-robin) updates adopted in \optimas{} do not guarantee global optimality in non-convex problems, but this does not constitute a flaw unique to our method; rather, it reflects a standard limitation of non-convex optimization broadly, and our global convergence guarantees only hold under additional structural assumptions, such as Polyak–Łojasiewicz or Kurdyka–Łojasiewicz conditions.}


%% file: chapters/5.1.experiment_setup.tex
\input{tables/main_results}

\section{Experiments}
\label{sec:exp}

We firstly summarize the datasets, baselines, and metrics, with details provided in the appendix.

\xhdrnodot{Benchmarks \& Evaluation} (Appendix~\ref{app:dataset}).  
We evaluate \optimas{} on five real-world tasks:

\begin{itemize}[leftmargin=*, itemsep=2pt]
    \vspace{-8pt}
    \item \textbf{\amazon{}}~\citep{amazon}: A behavior-driven recommendation task based on Amazon products, evaluated using accuracy to measure if the predicted next item matches the ground-truth item.
    \vspace{-2pt}
    \item \textbf{\pubmed{}}~\citep{pubmedqa}: A clinical classification dataset derived from PubMed abstracts~\citep{pubmed}, evaluated by accuracy, defined as the proportion of predictions that exactly match the ground-truth labels.
    \vspace{-2pt}
    \item \textbf{\textsc{STaRK-Prime}}~\citep{stark}: A retrieval benchmark over semi-structured biomedical corpora, evaluated using Mean Reciprocal Rank (MRR).
    \item \textbf{\hotpot{}}~\citep{hotpotqa}: A multi-hop question answering dataset, evaluated using the F1 score between predicted and ground-truth answers.
    \vspace{-2pt}
    \item \textbf{\code{}}~\citep{bigcodebench}: The instruction split of BigCodeBench for self-verifying code generation, evaluated using pass rate.
\end{itemize}

\vspace{-5pt}
\xhdrnodot{Compound Systems} (Figure~\ref{fig:tasks}, \cf Appendix~\ref{app:system}). We design a compound system per benchmark with diverse and common patterns for agentic systems.
All systems are accessible in our code repository. 

\xhdrnodot{Baselines} (See Appendix~\ref{app:baselines} for details). We compare \optimas{} against five baselines: Unoptimized, REINFORCE~\citep{Williams92}, LLMSelector~\citep{model_selection}, Hierarchical Behavior Cloning (HBC)~\citep{hbc}, TextGrad~\citep{textgrad}, DSPy~\citep{dsp,mipro}. \rebuttal{Moreover, we provide a single LLM reference which prompts an LLM to complete the task directly. This reference is used for justifying our system design in the experimental setup. }

%% file: tables/main_results.tex
\begin{table*}[t]
\centering
\caption{Performances of each method on the compound systems. The best and second-best results in each column are highlighted. Relative improvement is computed with respect to the best baseline. }
\vspace{-5pt}
\label{tab:main}
\resizebox{\textwidth}{!}{
\begin{tabular}{lccccc}
\toprule
& \begin{tabular}{c}\amazon{} \\ Product Rec.\\(Acc.)\end{tabular} 
& \begin{tabular}{c}\pubmed{} \\Medical Analysis \\(Acc.)\end{tabular} 
& \begin{tabular}{c} \stark{}\\Complex Retrieval\\(MRR)\end{tabular} 
& \begin{tabular}{c}\hotpot\\RAG\\(F1)\end{tabular} 
& \begin{tabular}{c}\code{}\\Verified Code Gen.\\(Pass Rate)\end{tabular} \\
\midrule
\rebuttal{Single LLM} & 20.20\std{1.43} & 54.13\std{2.73} & 0.00\std{0.00} & 21.58\std{1.24} & 35.47\std{0.34}\\
Unoptimized         & 21.21\std{3.78}  & 57.46\std{0.75}&  40.73\std{0.64} & 33.80\std{1.51} & \cellcolor{gray!10}36.67\std{1.35} \\
\rebuttal{REINFORCE}   & \cellcolor{gray!10}21.89\std{2.65} & - & -& -& - \\
LLMSelector         & -                & \cellcolor{gray!10}67.93\std{0.09}        & -               & -               & -               \\
HBC                 & 21.55\std{2.07}  & 58.80\std{0.58}& 36.95\std{0.59}         & 21.16\std{0.97} & 27.78\std{2.08} \\
TextGrad            & 20.88\std{3.53}  & \replace{55.80\std{1.08}}{56.96\std{2.24}}& 41.31\std{1.67}         & 24.86\std{1.19} & 35.71\std{0.10} \\
DSPy                & 18.18\std{0.82}  & 60.26\std{0.40} & \cellcolor{gray!10}41.40\std{0.04} & \cellcolor{gray!10}44.90\std{0.32} & 33.81\std{2.75} \\
\optimas{}             & \cellcolor{teal!10}\textbf{24.24\std{0.82}} & \cellcolor{teal!10}\textbf{69.13\std{0.33}}  & \cellcolor{teal!10}\textbf{50.54\std{0.70}}
 & \cellcolor{teal!10}\textbf{50.48\std{1.48}} & \cellcolor{teal!10}\textbf{38.92\std{0.36}} \\

Rel. Improv. & \cellcolor{blue!10}14.3\% & 1.8\%\cellcolor{blue!10}& \cellcolor{blue!10} 22.1\%& \cellcolor{blue!10}12.4\% & \cellcolor{blue!10}9.0\% \\
\bottomrule
\end{tabular}
}
\end{table*}

\begin{table*}[t]
    \centering
    \caption{Number of equivalent runs on the entire systems (in thousands) by TextGrad, DSPy, and \optimas{} in Table~\ref{tab:main}. We control the optimization process of each to use comparable system runs.}
    \vspace{-5pt}
    \label{tab:data_efficiency}
    \resizebox{1.0\textwidth}{!}{
        \begin{tabular}{lcccccc}
        \toprule
        & \amazon{} 
        & \pubmed{} 
        & \stark{} 
        & \hotpot  
        & \code{} & Average \\
        \midrule
        TextGrad & 0.32 & 0.70 & 0.70 & 2.12  & 0.18 & 0.80 \\
        DSPy     & 0.24 & 0.66 & 0.66 & 2.09  & 0.28 & 0.79 \\
        \optimas{}  & 0.31 & 0.52 & 0.51 & 2.02  & 0.21 & 0.71 \\
        \bottomrule
        \end{tabular}
        \vspace{-20pt}
    }
\end{table*}

%% file: chapters/5.2.experiment_main.tex
\begin{figure}[h]
    \centering
    \begin{subfigure}[t]{1\textwidth}
        \centering
        \includegraphics[width=\textwidth]{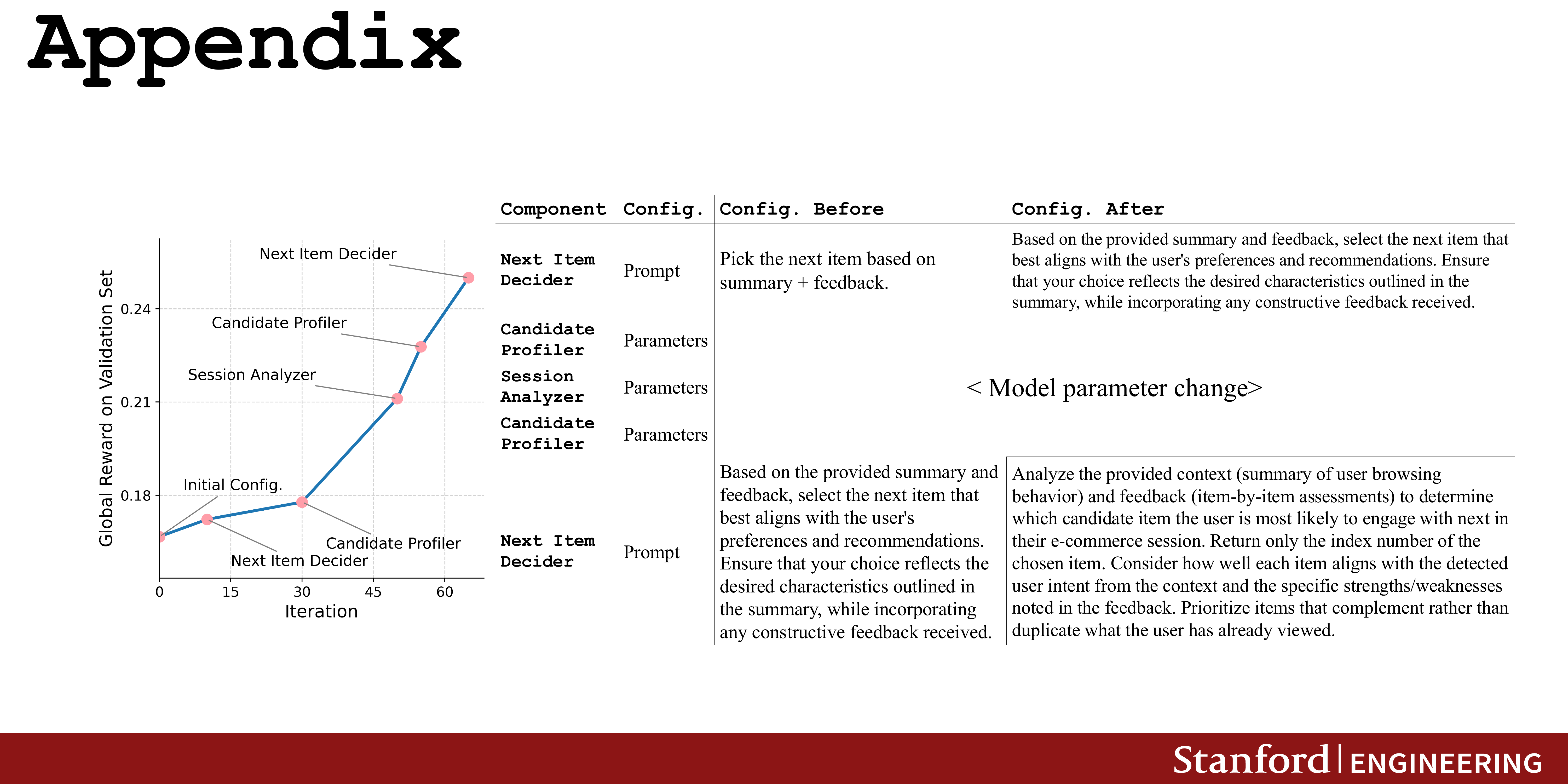}
        \vspace{-15pt}
        \caption{\amazon{} system }
        \vspace{7pt}
        \label{fig:dynamics_amazon_details}
    \end{subfigure}
    \hfill
    \begin{subfigure}[t]{1\textwidth}
        \centering
        \includegraphics[width=\textwidth]{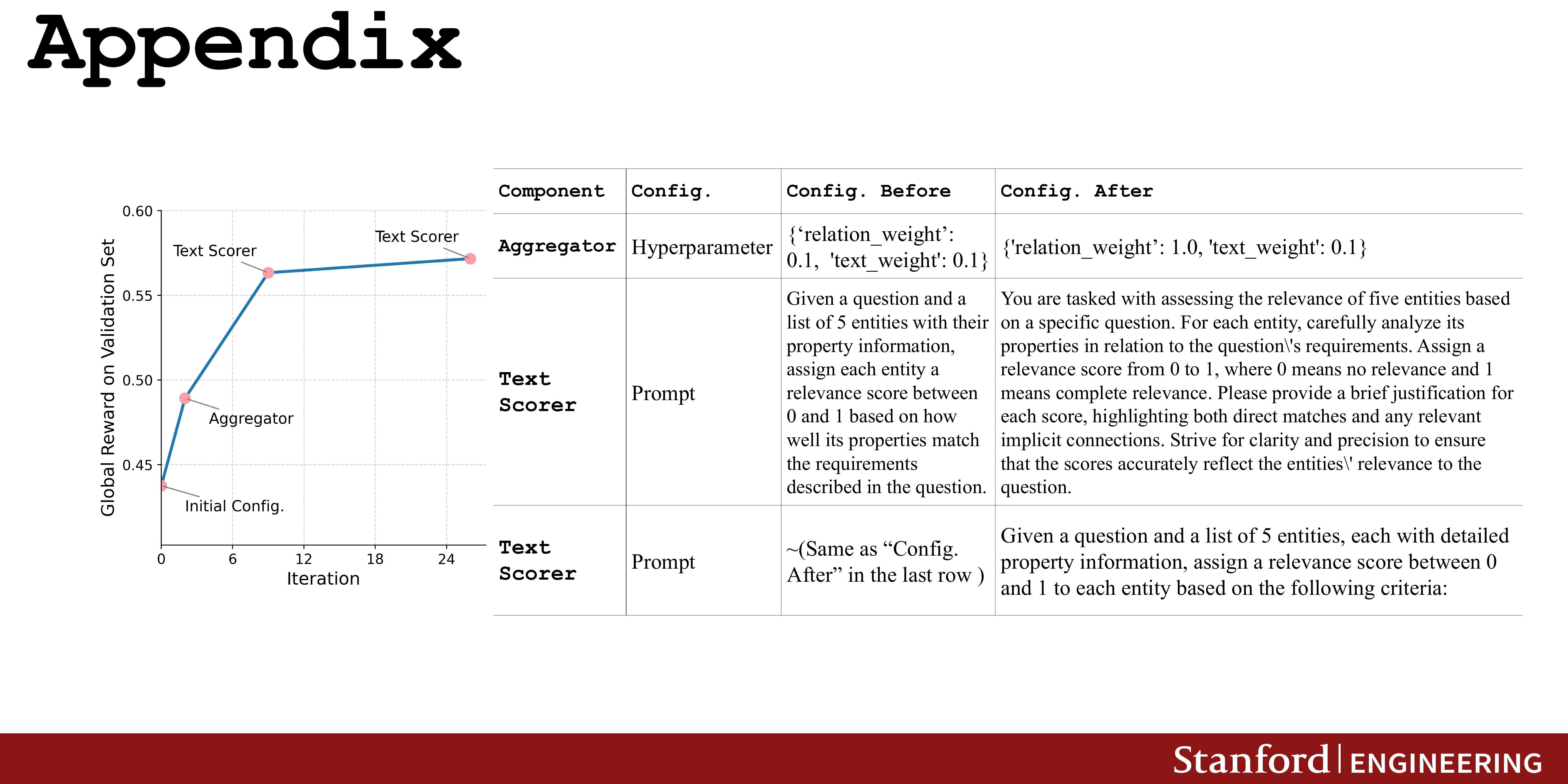}
        \vspace{-15pt}
        \caption{\stark{} system}
        \vspace{7pt}
        \label{fig:dynamics_stark_prime_details}
    \end{subfigure}
    \hfill
    \begin{subfigure}[t]{1\textwidth}
        \centering
        \includegraphics[width=\textwidth]{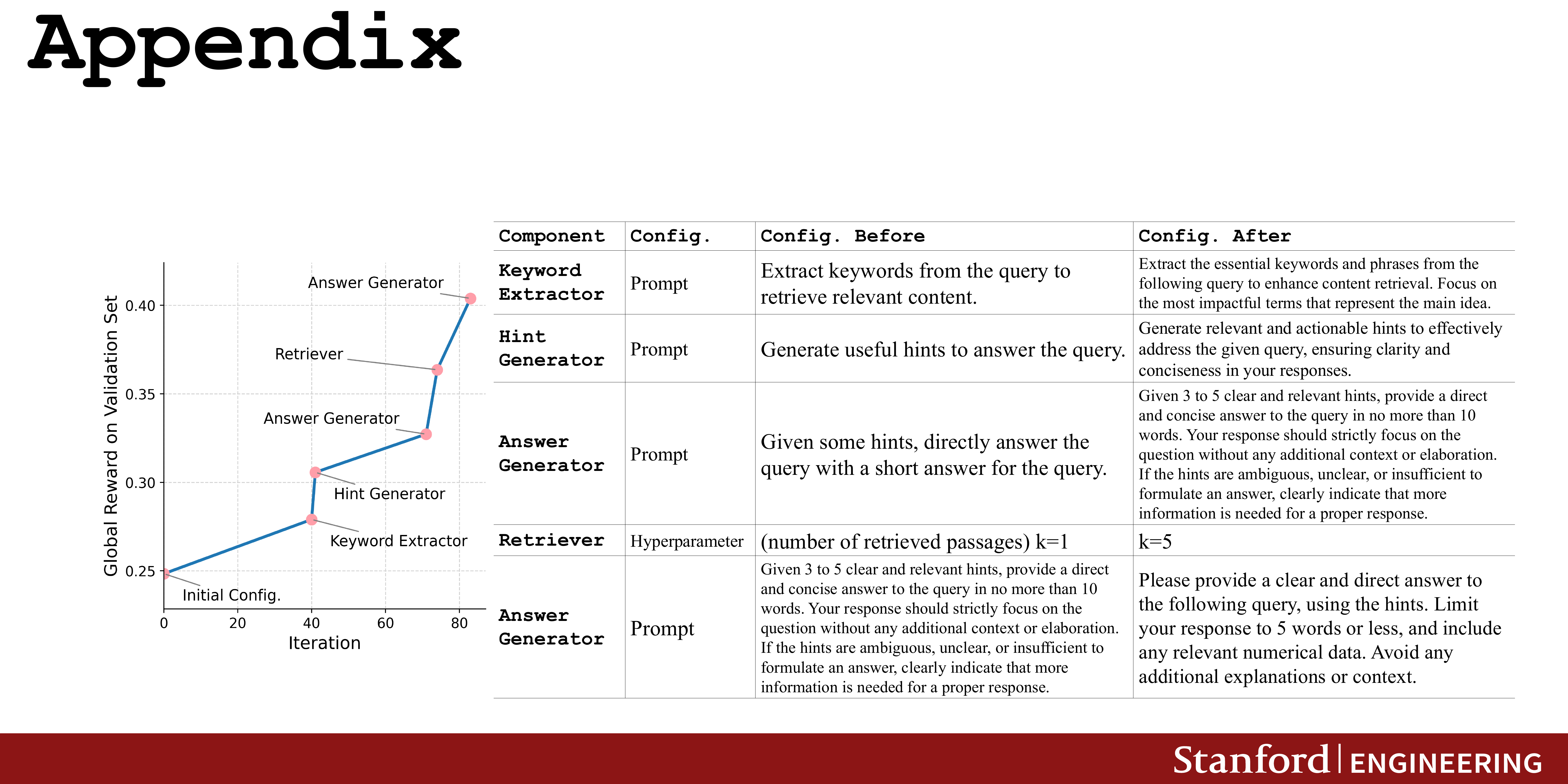}
        \vspace{-15pt}
        \caption{\hotpot{} system}
        \label{fig:dynamics_hotpotqa_details}
    \end{subfigure}
    \vspace{-5pt}
    \caption{Global reward and configuration updates of the three compound AI systems over the optimization iterations. For conciseness, we only show the local optimization steps that lead to an increase in \glob{} on the validation sets. The annotations show the optimized \obj{}s.
    }
    \label{fig:dynamics}
\end{figure}

\input{tables/pairwise_rank}

\vspace{-10pt}
\subsection{Performance of the systems during and after optimization.}

\xhdr{\iclr{Takeaway 1: \optimas{} leads to consistent and substantial improvements  (Table~\ref{tab:main})}}  We compare \optimas{} with the baselines under similar numbers of system runs. Under this controlled data cost (Table~\ref{tab:data_efficiency}), \optimas{} consistently improves \glob{}s across all compound systems, achieving an average relative improvement of \meanimprov\% compared to the best baseline. 

\rebuttal{For REINFORCE, while it yields performance gain from the unoptimized baseline, \optimas{} outperforms it by around 3\%. Moreover, REINFORCE needs collecting the reward signal from  downstream Monte Carlo sampling, which requires more than three times the data than \optimas{}}. While the strong baseline DSPy shows notable improvements on some datasets (\eg a 21.6\% gain on \hotpot{}), their performance may be inconsistent and can even degrade the system (\eg a 14.3\% drop on the \amazon{} dataset). For LLMSelector, it requires $2.8k$ times of forwarding through the entire system, which is $3x$ more expensive than \optimas{}. 

\xhdr{\iclr{Takeaway 2: Local optimization improves \glob{}s (Figure~\ref{fig:dynamics})}} 
We study how the configurations change in the local optimization. Within a small number of iterations, \optimas{} achieves a substantial average improvement of 41.7\% over the initial \glob{} on the validation sets, using lightweight and data-efficient updates. \iclr{Interestingly, we observe a mixed updates on prompts, model parameters, and hyperparameters, which can lead to improved \glob{}. For example, updating the prompt of \texttt{\small Text Scorer} in the 9-th iteration improves \glob{} from $0.49$ to $0.56$. Among these cases, they involve optimizing different components to achieve the highest \glob{} empirically, showing the importance of being able to optimize different types of configurations. }

\subsection{\iclr{Why and How \optimas{} Works: Alignment, Interpretability, and Efficiency}}

\begin{figure}[t]
  \centering

  \begin{subfigure}[t]{0.45\textwidth}
    \centering
    \includegraphics[width=\linewidth]{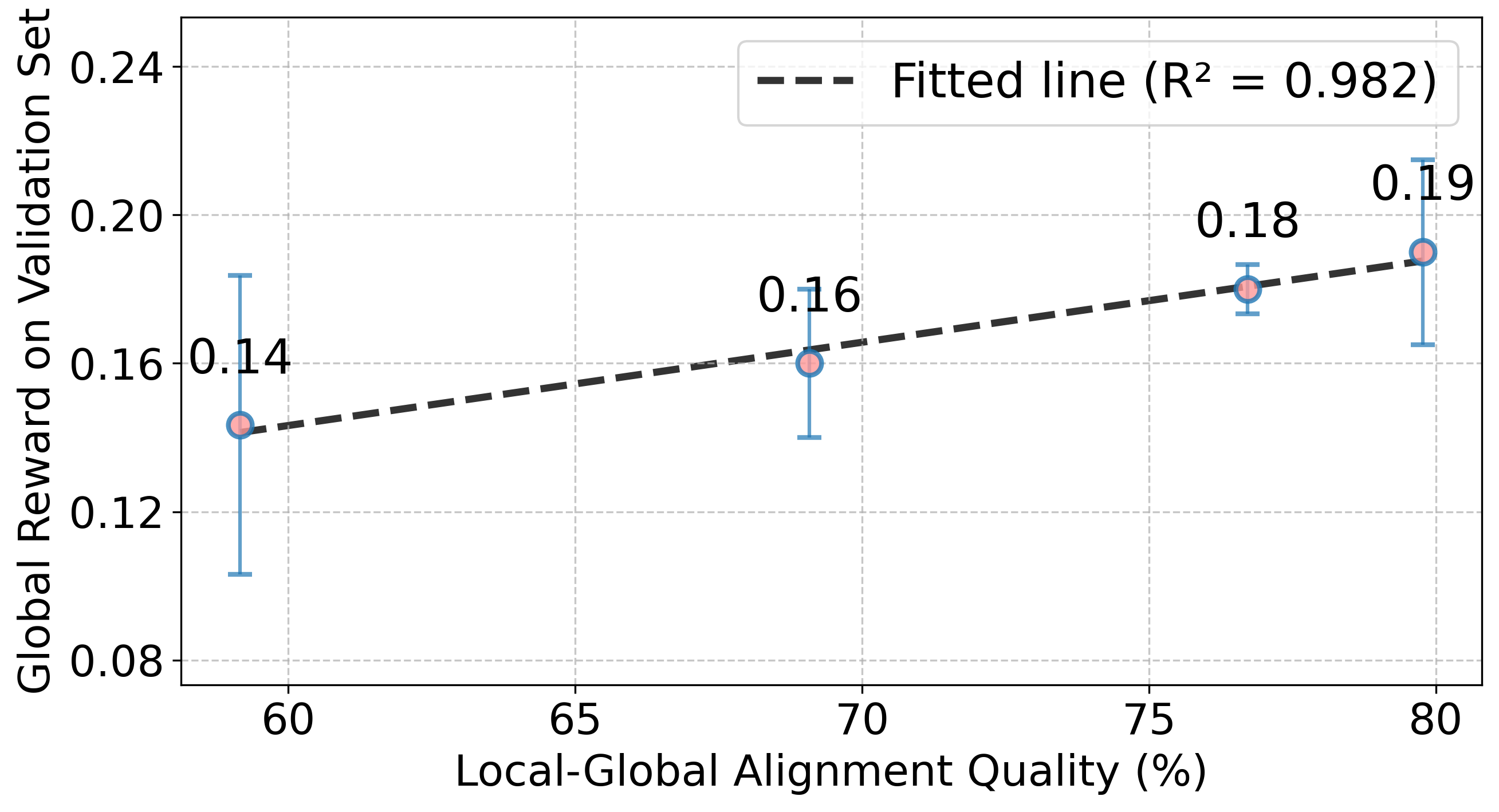}
    \vspace{-15pt}
    \caption{\texttt{\small Candidate Profiler} on \amazon{}}
    \label{fig:alignment_amazon}
  \end{subfigure}
  \hfill
  \begin{subfigure}[t]{0.45\textwidth}
    \centering
    \includegraphics[width=\linewidth]{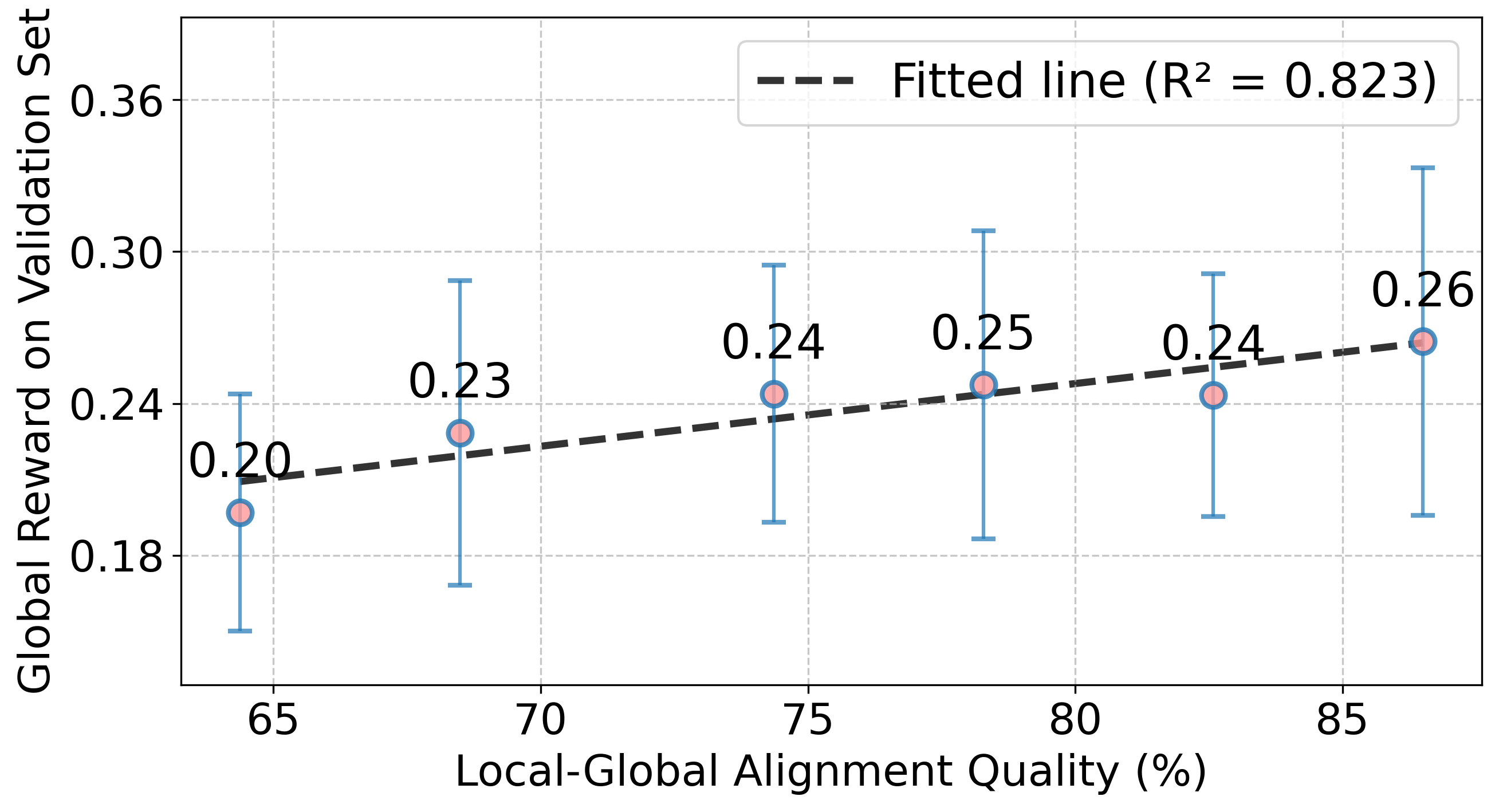}
    \vspace{-15pt}
    \caption{\texttt{\small Answer Generator} on \hotpot{}}
    \label{fig:alignment_hotpot}
  \end{subfigure}
  \hfill
    \vspace{-8pt}
  \caption{%
    Local reward models with varying alignment quality are used to optimize a selected \obj{} in each task, where we observe that higher alignment quality yields higher \glob{}s.  
    }
  \label{fig:alignment}
\end{figure}
\begin{figure}[t]
    \centering
    \includegraphics[width=1\textwidth]{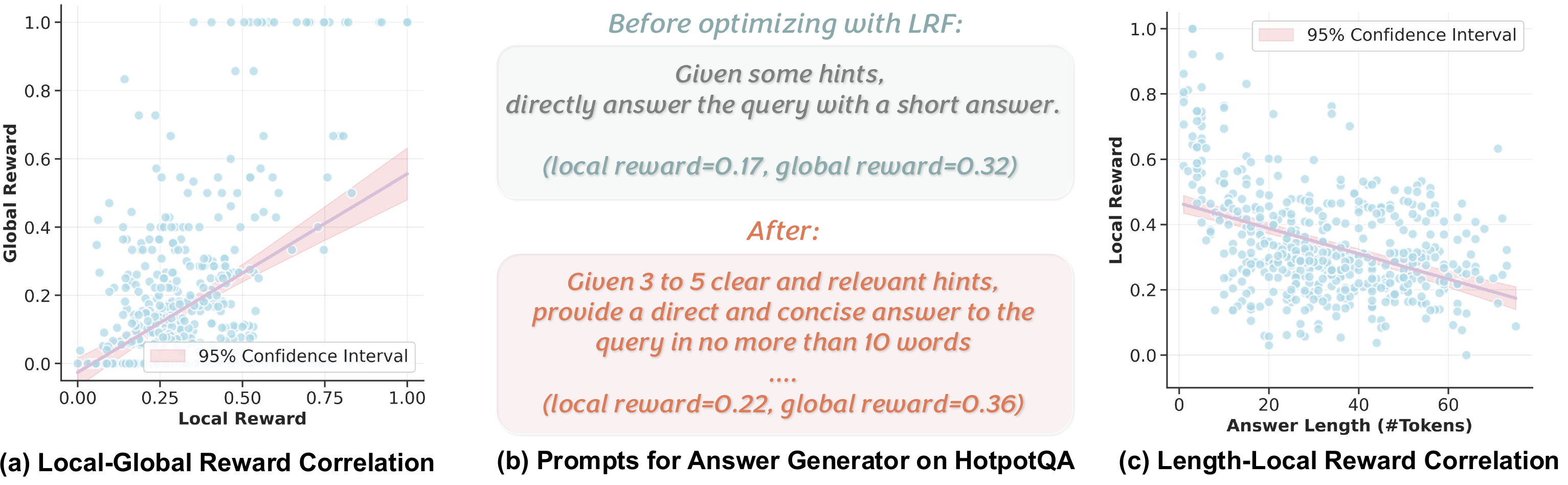}
    \vspace{-15pt}
    \caption{\iclr{An interpretability study on what is learned by LRFs. With (a) a well-aligned \abbr{}, we find that (b) the optimized prompt explicitly constrains the output length of the component. We attribute this to that (c) the \abbr{} prefers short outputs, which is consistent with the use of F1 as global metric.}}
    \label{fig:case_study_single_step}
    \vspace{-8pt}
\end{figure}

\iclr{We conduct extensive in-depth study to understand the mechanism in \optimas{} framework.}

\xhdr{\iclr{Takeaway 3: \optimas{} yields high local-global alignment quality (Table~\ref{tab:pairwise_rank})}}
\iclr{
To measure alignment quality of \abbr{}s, we compute pairwise ranking accuracy: the probability that an output with higher \glob{} receives a higher score than an output with lower \glob{}. This reflects how well the learned \abbr{}s aligns with \glob{}s.
We compare against a LLM Judge, which prompts a \texttt{gpt-4o} model to score the outputs of components based on 20 in-context examples. This approach is similar to prior methods such as TextGrad, which rely on few-shot reasoning over textual patterns.
In Table~\ref{tab:pairwise_rank}, LLM Judge performs closer to random guessing, due to the diversity and stochasticity of components' outputs that make it difficult to reason reliably.
In contrast, our \abbr{}s achieve substantially higher performances.
Moreover, \abbr{}s internalize the local-global alignment within their weights without relying on limited in-context examples, enabling more precise alignment. 
}


      




\iclr{
\xhdr{Takeaway 4: Higher alignment quality usually leads to higher \glob{} (Figure~\ref{fig:alignment})}
To understand how alignment quality affects \glob{}, we conduct controlled experiments, where we select a key \obj{}, apply \loc{} models with varying alignment quality to optimize it, and measure the average \glob{} achieved after updating the \obj{}. 
}
\iclr{
In Figure~\ref{fig:alignment}, we show a strong positive correlation between the \abbr{}'s alignment quality and the \glob{} improvement. 
}

\vspace{-5pt}
\xhdr{\iclr{Takeaway 5: \abbr{}s learn interpretable directions to improve \glob{}  (Figure~\ref{fig:case_study_single_step})}} 
\iclr{An important aspect of system optimization is interpretability, \ie if the configuration updates are reliable and understandable. We provide a study in Figure~\ref{fig:case_study_single_step}, where the \abbr{} learns to favor concise answers. In fact, it is more feasible to interpret \config{} updates with \abbr{}s. Specifically, one can perturb the component outputs in certain ways, and observe the changes in \loc{}s to obtain insights.
}

\begin{figure}[t]
    \centering
    \begin{minipage}[c]{0.55\textwidth}
        \centering
        \includegraphics[width=0.98\linewidth]{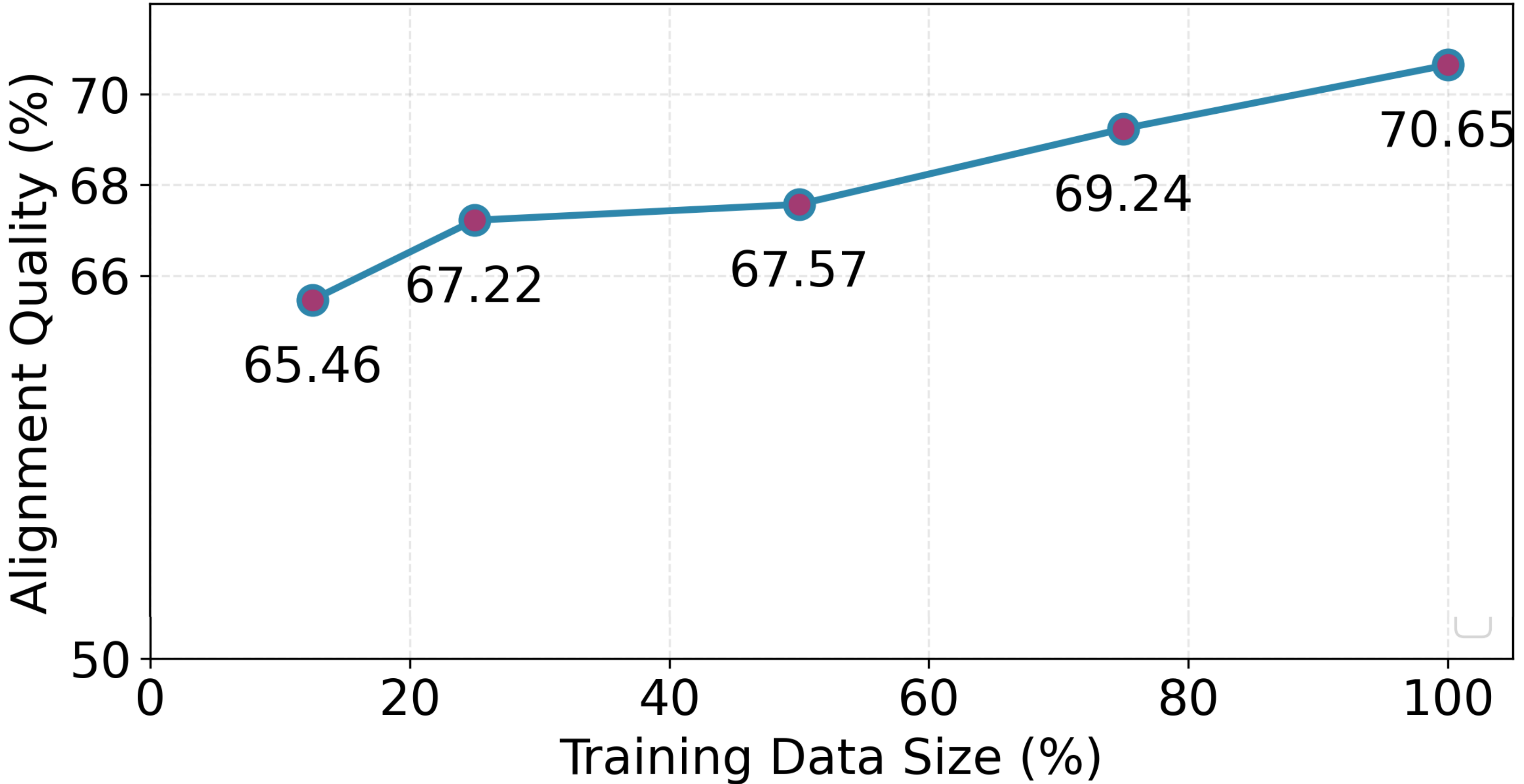}
    \end{minipage}
    \hfill
    \begin{minipage}[c]{0.44\textwidth}
        \centering
        \vspace{-10pt}
        \resizebox{0.75\textwidth}{!}{
        \begin{tabular}{c|c}
            \toprule
            Backbone  & Alignment  \\
             size & quality (\%) \\
            \midrule
            1B & 72.58 \\
            3B & 70.04 \\
            8B & 71.20 \\
            \bottomrule
        \end{tabular}
        }
    \end{minipage}
    \captionof{figure}{Impact of training data size and \abbr{} backbone size on alignment quality, measured by pairwise ranking accuracy: the percentage of preference pairs where local and global rewards agree.}
    \label{fig:efficiency}
    \vspace{-15pt}
\end{figure}

\xhdr{\iclr{Takeaway 6: \optimas{} is data- and computationally efficient. (Figure~\ref{fig:efficiency} \& Appendix~\ref{app:more_results})}} 
We study the efficiency of \optimas{} along two axes on the \hotpot{} system: (1)~how much training data is needed to learn effective \abbr{}s, and (2)~whether large \abbr{} backbones are necessary.

We train \abbr{} models using varying percentages of available training data (12.5\%--100\%) and measure alignment quality via pairwise ranking accuracy. Figure~\ref{fig:efficiency} (left) shows the performance degrades gracefully: using just 12.5\% of the data yields 65.46\% accuracy (92.7\% of full-data performance), and 25\% recovers 95.1\%. This indicates that \optimas{} is \textbf{data efficient}, where \abbr{} can be learned from relatively few system runs.
Moreover, Figure~\ref{fig:efficiency} (right) compares 1B, 3B, and 8B \abbr{} backbones and reports alignment quality. The results show that \optimas{} is \textbf{computationally efficient}, where lightweight models are sufficient for learning local-global alignment.

In Appendix~\ref{app:more_results}, we further show that local and global reward landscapes are closely aligned across retriever top-$k$ settings, that modest numbers of prompt candidates and adaptation inputs suffice for strong performance, and provide a cost comparison on the \amazon{} system showing that \optimas{} achieves the best performance while using a comparable number of effective system runs.

%% file: tables/pairwise_rank.tex
\begin{table*}[t]
\centering
\vspace{-5pt}
\caption{Average pairwise ranking accuracy on validation sets, measuring how often the method assigns a higher score to the output with higher expected \glob{}.}
\vspace{-5pt}
\label{tab:pairwise_rank}
\resizebox{\textwidth}{!}{
\begin{tabular}{lcccccc}
\toprule
& \amazon{} 
& \pubmed{} 
& \stark{} 
& \hotpot  
& \code{} & Avg. \\
\midrule
LLM Judge & 51.25\% & 49.54\% & 54.37\% & 50.00\% & 42.45\% & 49.52\% \\
\optimas{} & 84.93\% & 65.28\% & 76.64\% & 72.40\% & 90.57\% & \avgrankingacc\% \\
\bottomrule
\end{tabular}
}
\vspace{-5pt}
\end{table*}

%% file: chapters/7.conclusion.tex
\section{Conclusion}
\iclr{\optimas{} is a unified framework to optimize compound AI systems with heterogeneous configurations. \optimas{}' way to maintain globally aligned local reward functions allows every component, whether a fine-tunable LLM, LLM API, tools, or model selector, to be optimized locally while improving the overall system. 
On five real-world tasks, \optimas{} outperforms strong baselines, effectively optimizes components with different configurations, and exhibits high alignment quality and reliable interpretations. We believe \optimas{} will serve as a general, data-efficient approach for continually optimizing practical systems. For future work, we aim to further apply \optimas{} on even larger systems, with the goal of understanding complex reward modeling and scalability. }

\section*{Acknowledgements}
\vspace{-5pt}
We gratefully acknowledge the support of Amazon, 
NSF under Nos. CCF-1918940 (Expeditions), DMS-2327709 (IHBEM), IIS-2403318 (III);
NIH under No. 1U24NS146314-01,
Stanford Data Applications Initiative,
Wu Tsai Neurosciences Institute,
Stanford Institute for Human-Centered AI,
Chan Zuckerberg Initiative,
Genentech, SAP, and SCBX.

\section*{Ethics Statement}
\vspace{-5pt}
The applications of compound AI systems today include, but are not limited to, decision processes in socially and economically important domains, making their optimization an increasingly significant research problem. While our work aims to make optimization of such systems more data-efficient, robust, and interpretable, we recognize potential ethical concerns. In high-stakes settings such as healthcare or recommendation platforms, optimization could inadvertently amplify biases or propagate unsafe behaviors from individual components; our framework mitigates these risks by aligning local rewards with global performance, thereby reducing the likelihood of harmful emergent behaviors and creating opportunities to incorporate fairness- and safety-sensitive objectives. Another concern is malicious use of optimized compound systems; however, our approach relies on safety-aligned LLMs and tools, and we observe that aligning local and global objectives does not reduce the effectiveness of existing safeguards, while the decentralized optimization structure may even increase opportunities to detect misuse. Regarding data, our experiments rely solely on publicly available datasets or synthetic user interactions without any personally identifiable information (PII), and no private or sensitive user data were used. We believe this work contributes to the safe and responsible advancement of AI by providing both theoretical and practical insights into optimizing compound AI systems, and to promote further research in this direction we release all code, models, and benchmarks described in this paper.

\section*{Reproducibility statement} 
\vspace{-5pt}
We provide full open access to our implementation code, including the original datasets, compound system implementations, and the \optimas{} framework.
To support reproducibility, we provide Appendix~\ref{app:dataset} (Dataset Details), Appendix~\ref{app:system} (System Details), Appendix~\ref{app:baselines} (Baseline Details), and finally, Appendix~\ref{app:opt_details} (\optimas{} Details). 

%% file: chapters/9.appendix.tex
\newpage
\section{Algorithm}
\label{app:algorithm}

\begin{algorithm}[H]
\caption{Component-wise optimization with reward-model adaptation}
\begin{algorithmic}[1]
  \Statex \textbf{Input}
  \Statex $\mathcal{C}=\{C_1,\ldots,C_K\}$ \Comment{components}
  \Statex $\mathbf v^{0}=(\mathbf v^{0}_1,\ldots,\mathbf v^{0}_K)$ \Comment{initial configuration policy}
  \Statex $\Theta^{0}=\{\theta^{0}_{1},\ldots,\theta^{0}_{K}\}$ \Comment{parameters of local reward functions}
  \Statex Global reward $R(\cdot)$, preference dataset size $k$, total iterations $T$
  \Statex Training dataset $\mathcal{D}$ and validation dataset $D_v$

  \State $\mathbf v^{*} \gets \mathbf v^{0}$ 
  \State $\Theta^{t} \gets \Theta^{0}$

  \Statex \textbf{Optimization loop}
  \For{$t = 0,\ldots, T-1$}

    \Statex \textbf{Scheduler: choose a component to optimize}
    \State $i_t \sim \mathrm{Uniform}\big(\{1,\ldots,K\}\big)$ 

    \Statex \textbf{Local optimization for the chosen component $C_{i_t}$}
    \If{$C_{i_t}$ is an LLM with prompts}
      \State $\tilde{\mathbf{v}}^{\,t+1}_{i_t} \gets \textsc{PromptOptimization}(\mathbf{v}^{t}_{i_t}, \theta^{t+1}_{i_t})$
    \ElsIf{$C_{i_t}$ has trainable weights}
      \State $\tilde{\mathbf{v}}^{\,t+1}_{i_t} \gets \textsc{PPOTrain}(\mathbf{v}^{t}_{i_t}, \theta^{t+1}_{i_t})$
    \ElsIf{$C_{i_t}$ has a hyperparameter configuration}
      \State $\tilde{\mathbf{v}}^{\,t+1}_{i_t} \gets \textsc{HyperparameterSearch}(\mathbf{v}^{t}_{i_t}, \theta^{t+1}_{i_t})$
    \EndIf

    \State $\tilde{\mathbf v} \gets \mathbf v^{t}$ with $\mathbf{v}^{t}_{i_t}$ replaced by $\tilde{\mathbf{v}}^{\,t+1}_{i_t}$

    \Statex \textbf{Validation }
    \If{$\sum_{x_v \in D_v} R\!\big(x_v, f(x_v; \tilde{\mathbf v})\big) \;>\; \sum_{x_v \in D_v} R\!\big(x_v, f(x_v; \mathbf v^{t})\big)$}
      \State $\mathbf v^{t+1} \gets \tilde{\mathbf v}$; \quad $\mathbf v^{*} \gets \mathbf v^{t+1}$
          
        \State \textbf{Reward-model adaptation}
        \State $D_t \gets \textsc{CollectPreferenceData}(\mathcal{D}, \Theta^{t}, k)$ \Comment{create $k$ $(x^{+},x^{-})$ pairs }
        \State $\Theta^{t+1} \gets \textsc{RewardModelTrain}(\Theta^{t}, D_t)$
    \Else
      \State $\mathbf v^{t+1} \gets \mathbf v^{t}$\State $\mathbf \Theta^{t+1} \gets \mathbf \Theta^{t}$
    \EndIf

  \EndFor

  \State \Return $\mathbf v^{*}$
\end{algorithmic}
\end{algorithm}

\section{Theoretical Analysis}
\label{app:theory}
\paragraph{Formal statement of and proof of Theorem~\ref{thm:alignment}}

According to the procedure described in Section~\ref{sec:method}, the positive and negative pairs are determined by comparing the expected task metrics. We assume the estimated metrics at $y_k^{+}$ and $y_k^{-}$ are chosen following
$\mathbb P(\text{$y_k^+$ is labeled as positive})=\sigma_{\alpha}(\mathbb{E}_{\text{downstream}}
\Bigl[ R\bigl(x, f(x;\mathbf v_{-k}(x))\bigr) \,\big|\, y_k^{+}\Bigr]-\mathbb{E}_{\text{downstream}}
\Bigl[ R\bigl(x, f(x;\mathbf v_{-k}(x))\bigr) \,\big|\, y_k^{-}\Bigr])$, where $\sigma_\alpha(u)=\frac{1}{1+\exp(-\alpha u)}$ is the sigmoid function with parameter $\alpha>0$. $\alpha=+\infty$ corresponds to the case where the pairs are chosen deterministically. 

\begin{thm}
Under the conditions specified above, the maximizer of Eq. ~\ref{eq:reward_model} satisfies the local-global alignment property Eq. ~\ref{eq:alignment}. In addition, maximizing $r_k(x, C_k(x_k;\mathbf{v}_k))$ over $v_k$ and maximizing $R(x,f(x;\mathbf v_{-k})\mid C_k(x_k;\mathbf{v}_k))$ over $v_k$ will yield the same solution. 
\end{thm}

\begin{proof}
We first present a lemma.
\begin{lem}
    Suppose $(\bm x,y)\in\mathbb{R}^p \times \{-1,1\}$ follows the distribution $\Prob(y=1\mid \bm x)=\sigma_1(p^*(x))$ for some function $p:\bR^p\to(0,1)$. Then $$
\arg\rebuttal{\text{max}}_{p} \E[\log(\sigma_1(y\cdot p(x)))]=p^*.
    $$
\end{lem}
\begin{proof}
    We take the derivative of the left-hand side with respect to $p$ and set it to $0$: $$\E[\frac{\sigma_1(y\cdot p(x))\cdot \sigma_1(-y\cdot p(x))}{\sigma_1(y\cdot p(x))}\cdot y]=0,$$
    which is equivalent to $$
\E[\sigma_1(-y\cdot p(x))\cdot y]=0.
    $$
    We then have $$
    \sigma_1(-p(x)) \sigma_1(p^*(x))- \sigma_1(p(x))\sigma_1(-p^*(x))=0,
    $$
    and therefore $p(x)=p^*(x)$.
\end{proof}

Applying this lemma, we can obtain the solution to Eq. ~\ref{eq:reward_model} is $\alpha\cdot \mathbb{E}_{\text{downstream}}
\Bigl[ R\bigl(x, f(x;\mathbf v_{-k}(x))\bigr) \,\big|\, y_k\Bigr]$ for some positive $\alpha$, and therefore it satisfies the local-global alignment property Eq. ~\ref{eq:alignment}.

In the following, we then prove that maximizing $r_k(x, \pi_k(x))$ over $v_k$ and maximizing $R(x,f(x;\mathbf v_{-k})\mid \pi_k(x))$ over $v_k$ will yield the same solution. 
\begin{lem}Assume the local--global alignment property Eq. ~\ref{eq:alignment} holds for
component \(C_k\).  
Let \(\mathbf v(x)\) and \(\tilde{\mathbf v}(x)\) be two configuration
policies that differ \emph{only} in the policy for component \(C_k\);
denote the corresponding local outputs by \(y_k\) and \(\tilde y_k\),
respectively.  
If
\[
\mathbb{E}_x\!\bigl[r_k\bigl(x_k,\tilde y_k\bigr)\bigr]
\;>\;
\mathbb{E}_x\!\bigl[r_k\bigl(x_k, y_k\bigr)\bigr],
\]
then
\[
\mathbb{E}_x\!\Bigl[R\bigl(x,f(x;\tilde{\mathbf v}(x))\bigr)\Bigr]
\;\;\ge\;\;
\mathbb{E}_x\!\Bigl[R\bigl(x,f(x;\mathbf v(x))\bigr)\Bigr].
\]
\end{lem}
\begin{proof}
Fix an arbitrary input instance \(x\).
Because the two policies differ only at \(C_k\),
all other component configurations remain the same,  
so we can write
\[
f\!\bigl(x;\mathbf v_{-k}(x), y_k\bigr)
\quad\text{and}\quad
f\!\bigl(x;\mathbf v_{-k}(x),\tilde y_k\bigr),
\]
where \(\mathbf v_{-k}(x)\) denotes downstream configurations
(independent of the choice at \(C_k\)).  
By assumption on the expected local reward, we have
\(r_k(x_k,\tilde y_k)\ge r_k(x_k, y_k)\) for almost every \(x\).
Applying the alignment property Eq. ~\ref{eq:alignment} pointwise yields
\[
\mathbb{E}_{\text{downstream}}
\bigl[
  R\bigl(x,f(x;\mathbf v_{-k}(x),\tilde y_k)\bigr)
  \,\big|\, x_k
\bigr]
\;\;\ge\;\;
\mathbb{E}_{\text{downstream}}
\bigl[
  R\bigl(x,f(x;\mathbf v_{-k}(x), y_k)\bigr)
  \,\big|\, x_k
\bigr].
\]
Taking the expectation over \(x\) (law of total expectation) gives
\[
\mathbb{E}_x\!\Bigl[
  R\bigl(x,f(x;\tilde{\mathbf v}(x))\bigr)
\Bigr]
\;\;\ge\;\;
\mathbb{E}_x\!\Bigl[
  R\bigl(x,f(x;\mathbf v(x))\bigr)
\Bigr].
\]
Hence increasing the expected local reward for \(C_k\) cannot
decrease---and may strictly increase---the expected global objective.
\(\hfill\Box\)
\end{proof}
\end{proof}

\paragraph{Proof of Theorem~\ref{thm:convergence}} We first show that the algorithm is essentially performing coordinate maximization on $l(\mathbf v)$. In fact, given a previous configuration \(\mathbf{v}^t\), at time $t$, the updated configuration \(\mathbf{v}^{t+1}\) only changes the configuration of a single component, say, $C_k$. As the change is solved by maximizing $r_k(x, C_k(x_k;\mathbf{v}_k))$, by Theorem~\ref{thm:alignment}, this is equivalently maximizing $l(\mathbf v_k, \mathbf v_{-k})$ for the $k$-th coordinate.

To rule out cycling, we first prove that the discrete block $\mathbf v^{(2)}$ stabilizes. Consider the sequence
 $(\mathbf v^{1,(1)},\mathbf v^{1,(2)}), (\mathbf v^{2,(1)},\mathbf v^{2,(2)}),...,(\mathbf v^{t,(1)},\mathbf v^{t,(2)}),...$. We first show that there exists $T>0$, such that for all $k>0$, $\mathbf v^{T+k,(2)}=\mathbf v^{T, (2)}$. 

As $\mathbf v^{(2)}$ are discrete, there are finitely many different configurations. In addition, according to the assumption that the coordinate-wise maximum is unique, each update will result in a strict decrease in the loss function. Therefore, after finite number of iterations, $\mathbf v^{(2)}$ will not change.

Now when we consider all the iterations that are later than the time $T$, $\mathbf v^{(2)}$ is fixed, and we only need to consider the update regarding $\mathbf v^{(1)}$. In this case, we apply Theorem 4.1 of \citep{tseng2001convergence} and complete the proof. 


\section{Dataset Details}
\label{app:dataset}
This appendix provides the essential statistics and source information for the five compound‑system benchmarks used in our experiments (Figure~\ref{fig:tasks}).  
For each dataset we specify the train / validation / test split sizes, the task formulation as used in the compound pipeline, and the evaluation metric reported in the main paper

\paragraph{\amazon{} (Behavior‑Driven Next‑Item Recommendation).}
The corpus is derived from Amazon MMLU dataset~\citep{amazon}.  
Each instance consists of a user’s historical behaviour sequence (views, clicks, purchases) and the target “next item’’ to be recommended. We split it into to \texttt{335} / \texttt{60} / \texttt{99} user sequences after filtering malformed entries. Accuracy - whether the predicted item number 1 matches the ground truth - is the evaluation metric.

\paragraph{\pubmed{} (Medical Analysis based QA).}
\pubmed{}~\citep{pubmedqa} contains biomedical abstracts paired with yes/no/maybe answers to research questions. We keep the original “expert’’ split and discard ambiguous samples, resulting in \texttt{475} / \texttt{25} / \texttt{500} question–abstract pairs. Our compound system frames the task as three‑way classification; exact‑match accuracy is reported. 

\paragraph{\stark{} (Semi-Structured Knowledge Base Retrieval).}
\stark{}‑Prime originates from STARK benchmark introduced by~\citep{stark}.  It blends free‑text passages with relational triples from biomedical knowledge graphs.  
Queries are natural‑language questions; relevance labels are automatically propagated from the original STARK annotations.  
We uses the original dataset split: \texttt{495} / \texttt{51} / \texttt{96} queries. Performance is measured by Hit@1, which is the rate of ranking the ground truth items in the predicted ranking list.  

\paragraph{\hotpot{} (Retrieval-Augmented Multi-Hop QA).}
We adopt the \hotpot{}~\citep{hotpotqa} and keep the official train/dev/test splits: \texttt{1000}, \texttt{250}, and \texttt{100} questions respectively.  Each example in the  set contains  a question and its (human-annotated) answer. We report answer‑level F1 score.

\paragraph{\code{} (Self-Verified Code Generation).}
We use a subset of the \textit{full-instruction} subset of BigCodeBench~\citep{bigcodebench} due to efficiency issue.
After proportionally drop the data, we obtain \texttt{500} / \texttt{25} / \texttt{70} coding tasks.  Each sample includes a natural‑language specification and reference unit tests.  
Our metric is \textit{pass@1}: the proportion of generated programs that pass all tests in one try.

\section{Compound AI System Details}
\label{app:system}

Table~\ref{tab:modules-optimization} summarizes each pipeline’s modules (columns: \emph{System}, \emph{Module}, \emph{Model}, \emph{Config}, and \emph{Optimization}). In the table below, we clarify the various configuration spaces and optimization methods used across the five systems.

\begin{table}[h!]
\centering
\caption{Modules, models, and optimization methods. \textit{Model Selection (LLMs)\protect\footnotemark[1]} indicates a discrete choice of LLMs. 
\textit{Aggregator}\protect\footnotemark[2] is tuned over coefficients \(\texttt{relation\_weight}\) and \(\texttt{text\_weight}\). 
\textit{Retriever}\protect\footnotemark[3] has a hyperparameter \(\texttt{k}\).}
\label{tab:modules-optimization}
\resizebox{\textwidth}{!}{
\begin{tabular}{llccc}
\toprule
\textbf{System} & \textbf{Module} & \textbf{Model} & \textbf{Config} & \textbf{Optimization} \\
\midrule
\multirow{3}{*}{\textbf{Amazon}} 
 & Session Analyzer    & Qwen~2.5\,1.5B  & Model Params & PPO (RL) \\
 & Candidate Profiler  & Qwen~2.5\,1.5B  & Model Params & PPO (RL) \\
 & Next Item Decider   & GPT-4o-mini   & Prompt       & Prompt Opt.\\
\midrule
\multirow{4}{*}{\textbf{PubMed}} 
 & Context Model Selector & --                             & Model Selection (LLMs)\footnotemark[1] & Hyperparam Search \\
 & Context Analyst        & One of \{\texttt{gpt-4o}, \ldots\}\footnotemark[1] & Prompt & Prompt Opt. \\
 & Solver Model Selector  & --                             & Model Selection (LLMs)\footnotemark[1] & Hyperparam Search \\
 & Problem Solver         & One of \{\texttt{gpt-4o}, \ldots\}\footnotemark[1] & Prompt & Prompt Opt.\\
\midrule
\multirow{3}{*}{\textbf{STaRK}} 
 & Text Scorer     & Claude 3 Haiku & Prompt & Prompt Opt. \\
 & Relation Scorer & Claude 3 Haiku & Prompt & Prompt Opt. \\
 & Aggregator\footnotemark[2]
                   & --             & Coefficients & Hyperparam Search \\
\midrule
\multirow{5}{*}{\textbf{HotpotQA}} 
 & Question Rewriter & GPT-4o-mini   & Prompt & Prompt Opt. \\
 & Info Extractor    & GPT-4o-mini   & Prompt & Prompt Opt. \\
 & Retriever\footnotemark[3]
                    & --            & \#Retrieved passages & Hyperparam Search \\
 & Hint Generator    & GPT-4o-mini   & Prompt & Prompt Opt. \\
 & Answer Generator  & GPT-4o-mini   & Prompt & Prompt Opt. \\
\midrule
\multirow{3}{*}{\textbf{BigCodeBench}} 
 & Code Generator       & Claude 3 Haiku & Prompt & Prompt Opt. \\
 & Unit Test Generator  & Claude 3 Haiku & Prompt & Prompt Opt. \\
 & Final Code Generator & Claude 3 Haiku & Prompt & Prompt Opt. \\
\bottomrule
\end{tabular}
}
\end{table}

\footnotetext[1]{\textbf{Model Selection (LLMs)}: we search over 
\{\texttt{gpt-4o}, \texttt{gpt-4o-mini}, \texttt{gpt-3.5-turbo-0125}, \texttt{gpt-4-turbo}, \texttt{claude-3-5-haiku-20241022}, \texttt{claude-3-5-sonnet-20241022}, \texttt{claude-3-7-sonnet-20250219}\}.}
\footnotetext[2]{\textbf{Aggregator (STaRK)}: we search \(\texttt{relation\_weight}, \texttt{text\_weight}\in\{0.1,1.0\}\).}
\footnotetext[3]{\textbf{Retriever (HotpotQA)}: we search \(\texttt{k}\in\{1,5,10,25\}\).}

\paragraph{\amazon{} (Behavior-Driven Next-item Recommendation).} \textit{Session Analyzer} and \textit{Candidate Profiler} both use the Qwen~2.5\,1.5B model; we optimize their model parameters with PPO reinforcement learning \citep{ppo}. This helps each module better encode task-specific knowledge, \ie user sessions and product candidates. The final \textit{Next Item Decider} is a GPT-4o-mini module, whose \textit{prompt} we optimize.

\paragraph{\pubmed{} (Medical Analysis based QA).} Two modules (\textit{Context Model Selector} and \textit{Solver Model Selector}) each do discrete model selection from a \emph{list of possible LLMs}. At inference time, these selectors use a reward model to pick the best LLM for each input instance. The \textit{Context Analyst} and \textit{Problem Solver} modules then receive the chosen model and optimize \emph{only the prompt} for improved medical QA performance.

\paragraph{\stark{} (Semi-Structured KB Retrieval).} We have two scoring modules (\textit{Text Scorer}, \textit{Relation Scorer}), both using Claude~3~Haiku with \emph{prompt optimization}. The \textit{Aggregator} merges these two scores; we tune two numeric weights, \(\texttt{relation\_weight}\) and \(\texttt{text\_weight}\), each set in \(\{0.1,1.0\}\). We perform a global hyperparameter search across the entire training set for the \textit{Aggregator} module, tuning 
\(\texttt{relation\_weight}\)and \(\texttt{text\_weight}\). After identifying the best fixed combination, we use it in for inference.

\paragraph{\hotpot{} (Retrieval-Augmented Multi-Hop QA).} Four GPT-4o-mini modules (\textit{Question Rewriter}, \textit{Info Extractor}, \textit{Hint Generator}, \textit{Answer Generator}) each rely on \emph{prompt optimization} to improve multi-hop reasoning. Meanwhile, the \textit{Retriever} is a retriever with a key hyperparameter \(\texttt{k}\), the number of passages to pull. We search \(\texttt{k}\in\{1,5,10,25\}\) which we also tune via global hyperparameter search across  training instances.

\paragraph{\code{} (Self-Verified Code Generation).} All three modules (\textit{Code Generator}, \textit{Unit Test Generator}, \textit{Final Code Generator}) are Claude~3~Haiku LLMs; each uses \emph{prompt optimization} to iteratively refine code solutions based on test outcomes. The global objective is a higher pass rate on the final code.

Overall, the table illustrates how \textbf{different modules can require different types of optimization}---from prompt tuning (textual modifications) and model-parameter fine-tuning (PPO) to discrete model selection and hyperparameter search. By unifying these heterogeneous updates within our \optimas{} framework, we effectively coordinate local improvements to achieve consistent global reward gains.

All experiments were run on a node with 8 NVIDIA A100 GPUs (80\,GB memory each); depending on the complexity of the compound system and hyperparameters, training and optimization typically finished in 2--8 hours.

\begin{table}[h]
\centering
\vspace{-10pt}
\caption{Key hyperparameters for the local reward model training.}
\label{tab:rm-training-hparams}
\label{tab:lora-hparams}
\begin{tabular}{l l}
\toprule
\textbf{Parameter} & \textbf{Value} \\
\midrule
Base model & Llama 3 8B Instruct \\
LoRA rank  & 32 \\
LoRA alpha  & 16 \\
Maximum sequence length & 2048 tokens \\
Learning rate  & 2e-6 \\
Number of epochs  & 25 \\
batch size & 32 \\
\bottomrule
\end{tabular}
\end{table}

Before on-policy optimization, we train a reward model on preference pairs (``chosen'' vs.\ ``rejected'' system outputs) so it can assign higher scores to better outputs. 
Table~\ref{tab:rm-training-hparams} highlights the main hyperparameters for this stage. 
We adopt LoRA for memory efficiency, and use an early-stopping mechanism based on the evaluation loss (patience=512 steps) to reduce overfitting.

After training the reward model, we run iterative on-policy optimization for each module in the compound AI system. 
Table~\ref{tab:onpolicy-hparams} lists the key hyperparameters. For example, the \emph{train size} (50) limits how many examples are used to train each module’s local configuration, while the \emph{search size} (50) sets how many samples we use when searching for the best local update. 
When a module is selected, we collect a small preference data, retrain (or adapt) the reward model as needed, then locally optimize that module’s parameters or prompts to maximize its local reward.

\begin{table}[h!]
\centering
\caption{Key hyperparameters for on-policy optimization.}
\label{tab:onpolicy-hparams}
\begin{tabular}{ll}
\toprule
\textbf{Parameter} & \textbf{Value / Description}\\
\midrule
Train size                & 50 \\
Search size              & 50 \\
Prompt candidates        & 3  \\
Local optimization steps & 3  \\
Fresh input size         & 20 \\
Validation size          & 20 \\
\bottomrule
\end{tabular}
\end{table}

\section{Baseline Details}
\label{app:baselines}

We provide the details to reproduce the reported baseline results and the results of \optimas{}.

\begin{itemize}[leftmargin=*, itemsep=2pt]
    \vspace{-8pt}
    \item \textbf{Unoptimized}: This system uses default settings for all components without any optimization.
    \vspace{-2pt}
    \item \textbf{LLMSelector}~\citep{model_selection}: A lightweight policy selects the best LLM per component via model routing, without updating other configurations. Only applicable on \pubmed{}. 
    We run LLMSelector  \citep{model_selection} with the \textsc{LLMDiagnoser} to estimate per-module performance, following their procedure. We perform two rounds of allocation updates with 100 training examples each round. 
    \vspace{-2pt}
    \item \rebuttal{\textbf{REINFORCE}~\citep{Williams92}: A policy-gradient baseline that directly updates the parameters of local LLM components to maximize the task reward. This method is only applicable on \amazon{}, where we deploy two locally hosted LLMs with trainable parameters. We optimize each component with REINFORCE using sampled trajectories from the full system and propagate the scalar reward back to the corresponding module-level policy. We use 16 rollouts per step to estimate the expected system performance, with learning rate $10^{-5}$ for 100 training steps. }

    \item \textbf{Hierarchical Behavior Cloning (HBC)}~\citep{hbc}: A hierarchical imitation learning method that optimizes \obj{}s to produce outputs similar to those that lead to high \glob{}s. \rebuttal{Specifically, we collect successful trajectories to approximate ground truth intermediate outputs, and then perform supervised updates on the local components to mimic/clone the ideal behavior.}

    We run HBC using the collected preference dataset by replacing the original reward model with the embedding similarity score. With the same input in the preference dataset, we use \texttt{text-embedding-3-small} to embed the module output and the preferred output in the preference dataset and calculate the embedding similarity score. We further weight the similarity score using the gap of the preferred output score and the rejected output score to get the reward for HBC.

    \vspace{-2pt}
    \item \textbf{TextGrad}~\citep{textgrad}: A gradient-based prompt tuning method using estimated gradients from black-box LLMs to improve prompt efficacy.

    We run TextGrad using \texttt{GPT-4o mini} to optimize each \obj{}'s prompt independently in separate epochs. Validation is performed every two optimization steps using 20 held-out validation instances. A batch size of 4 is used across all components and datasets. The best-performing prompt, as determined by validation frequency of LLM-generated textual feedback, is selected as the final configuration.
    \vspace{-2pt}
    \item \textbf{DSPy}~\citep{dsp,mipro}: A prompt optimization framework using the MIPRO algorithm that jointly refines module-level instructions and few-shot demonstrations. We conduct optimization on DSPy's MIPRO~\citep{mipro} prompt optimization approach. For fair comparison, we disable the few-shot example and system prompt optimization and only conducts optimization on the user instructions. We dynamically set the number of iterations for the  MIPRO optimizer to match the budget of system runs in Table~\ref{tab:data_efficiency}.
\end{itemize}


For TextGrad, DSPy, and \optimas{}, we consistently using the same 20 held-out validation instances on each dataset to select the best \config{}s.

\section{\optimas{} Details}
\label{app:opt_details}

\xhdr{\Obj{} selection}  
At each iteration \(t\), we randomly select a \obj{} to optimize. 



\xhdr{Local Optimization Steps for Different Configurations} 
Given a globally aligned \abbr{} \(r_k\), we perform local optimization on each \obj{} \(C_k\) to improve its \config{} \(\mathbf{v}_k\). 
Specifically, we solve:
\begin{equation}
\label{eq:config}
\mathbf{v}_k^{t+1}
=
\arg\max_{\mathbf{v}_k \in \mathcal{V}_k}
\;
\mathbb{E}_{x_k} \bigl[
    r_k\bigl(x_k,\,C_k(x_k;\mathbf{v}_k)\bigr)
\bigr]
\quad \text{subject to} \quad
d\bigl(\mathbf{v}_k, \mathbf{v}_k^{t}\bigr)\; \le\; \delta,
\end{equation}
where \(\mathbf{v}_k^{t}\) is the configuration before the $t$-th iteration, \(d(\cdot, \cdot)\) is a distance function over configurations, and \(\delta\) defines a trust region that bounds allowable updates. This constraint ensures that \(r_k\) is used within a region where it is expected to produce reliable evaluations.

In practice, explicitly setting the trust region threshold \(\delta\) can be difficult due to heterogeneous configuration types (\eg continuous weights or discrete tokens). Instead, we adopt a conservative number of update steps to restrict the magnitude of change during each iteration. 

\vspace{-5pt}
\begin{itemize}[leftmargin=*]
    \item \textbf{Prompt tuning.}  
    For textual prompts, we apply prompt optimization algorithms~\citep{llm_as_opt, avatar}, using \(r_k\) as the evaluation metric. We sample multiple prompts limited by a max number of prompt candidates. The prompts are ranked by average reward over validation instances, and the best-performing prompt is selected.

    \item \textbf{Model fine-tuning.}  
    When \(C_k\) is an LLM or neural model with trainable parameters, we can apply reinforcement learning algorithms, such as Proximal Policy Optimization (PPO)~\citep{ppo}, using \(r_k\) as the critic. The model parameters are updated for a small and fixed number of steps.

    \item \textbf{Model selection and hyperparameter tuning.}  
    For discrete or low-dimensional continuous configurations, such as model selection, tool routing, or scalar hyperparameters, we formulate the optimization as a sampling problem parameterized by a probabilistic distribution. Since these configurations are instance-specific, the expectation in Eq.~Eq. ~\ref{eq:config} reduces to a single input. For each input \(x\), we evaluate a set of candidate configurations using the \abbr{} \(r_k\), and compute a probability distribution over candidates proportional to \(\exp\{r_k(x_k, C_k(x_k; \mathbf{v}_k))\}\). This distribution is then used to sample the configuration update for the current iteration.
\end{itemize}

Under a conservative update to the \config{} of a \obj{} \(C_k\), the expected \glob{} is guaranteed to maintain or improve, if the local–global alignment property in Eq.~Eq. ~\ref{eq:alignment} holds.

\vspace{-5pt}

\section{More Experiment Results}
\label{app:more_results}

\subsection{Local vs.\ Global reward landscapes}

\begin{table}[h]
\centering
\caption{On \hotpot{}, sweeping the retriever's top-$k$ reveals closely aligned local and global reward landscapes.}
\label{tab:hotpot_retriever_k}
\begin{tabular}{c|ccccccc}
\toprule
$k$ & 1 & 2 & 3 & 5 & 10 & 15 & 25 \\
\midrule
Local reward  & 0.4247 & 0.5578 & 0.5695 & 0.6124 & 0.6117 & 0.5949 & 0.5123 \\
Global reward & 0.3398 & 0.3493 & 0.3325 & 0.3598 & 0.3645 & 0.3568 & 0.3465 \\
\bottomrule
\end{tabular}
\end{table}
To better understand how local objectives reflect global performance, we sweep the retriever’s top-$k$ setting on \hotpot{} and compare local and global rewards (Table~\ref{tab:hotpot_retriever_k}). The two landscapes are closely aligned: both are unimodal and peak at nearby values ($k{=}5$ for the local reward and $k{=}10$ for the global reward). The top-3 configurations $\{5,10,15\}$ coincide, differing only in the order of the top-2. This alignment shows that the local reward provides a reliable proxy for the global objective, lending empirical support to our theoretical guarantee that local optimization drives system-level gains.




\subsection{Prompt candidates per step}

\begin{table}[h]
\centering
\caption{On \hotpot{}, we vary the number of candidate prompts per optimization step. We compute the global performance (F1) under each experiments.}
\label{tab:prompt_candidates_hotpot}
\begin{tabular}{c|cccc}
\toprule
\# candidates & 3 & 5 & 7 & 10 \\
\midrule
Final F1 & 0.2822 & 0.1945 & 0.2968 & 0.2405 \\
\bottomrule
\end{tabular}
\end{table}

Another factor is the number of candidate prompts considered at each optimization step. On \hotpot{}, we vary this number from 3 to 10 and measure the resulting global performance (Table~\ref{tab:prompt_candidates_hotpot}). The best F1 score is achieved with 7 candidates, although results with fewer candidates remain competitive. These findings indicate that exhaustive candidate pools are unnecessary, and that modest numbers already yield strong performance with lower computational cost.

\subsection{New inputs for \abbr{} adaptation}

\begin{table}[h]
\centering
\caption{On \hotpot{}, number of new inputs used to collect preference pairs for \abbr{} adaptation vs.\ final global performance (F1).}
\label{tab:new_inputs_hotpot}
\begin{tabular}{c|cccc}
\toprule
\# new inputs & 10 & 20 & 30 & 40 \\
\midrule
Final F1 & 0.2773 & 0.2822 & 0.2659 & 0.2533 \\
\bottomrule
\end{tabular}
\end{table}
Finally, we explore how many new inputs are needed when adapting the \abbr{}s. On \hotpot{}, we test adaptation with 10 to 40 new inputs (Table~\ref{tab:new_inputs_hotpot}). Performance improves up to about 20 inputs, after which gains plateau and even slightly decline. This suggests that effective adaptation can be achieved with relatively small amounts of new data, reinforcing the practicality of our approach in scenarios where data collection is limited.


\subsection{\rebuttal{Cost Estimation}}

\label{app:cost_estimation}

We report a
detailed breakdown of the cost on the \amazon{} system. This is the only
system where components have trainable local models and thus requires PPO
training. For all methods, we measure
cost in terms of \emph{full system runs}, \ie one invocation of the system and assume that all components contribute equally to the per-run cost. We label the three components on Amazon system as $A, B, C$ in topological order.
\begin{table}[h]
\centering
\caption{Cost--performance comparison on the \amazon{} system. Individual cost components are shown in actual runs, with total runs measured in thousands (k). ``LRF training'' and ``LRF adaptation'' only apply to \optimas{}, which learns trainable local reward functions. DSPy and TextGrad optimize only the prompt of the last component $C$, and we report their cost as an \emph{effective} number of full system runs (see text).}
\label{tab:amazon_cost}
\begin{tabular}{l|ccccc}
\toprule
Method & LRF training & LRF adaptation & Validation cost & Total runs (k) & Performance (\%) \\
\midrule
\optimas{}        & 120 & 100 & 93  & 0.31 & 24.24 \\
TextGrad       & \textemdash{} & \textemdash{} & 320 & 0.32 & 20.88 \\
DSPy           & \textemdash{} & \textemdash{} & 240 & 0.24 & 18.18 \\
\bottomrule
\end{tabular}
\end{table}

For \optimas{}, the total of $\approx 0.31$k effective full system runs decomposes
into three parts:
(i) \textbf{LRF training}: we collect 60 initial preference pairs per
component; each pair requires 2 full system runs, giving $60 \times 2 = 120$
runs;
(ii) \textbf{LRF adaptation}: during local optimization, we update the LRFs 5
times, each time collecting 10 new preference pairs per component, again with 2
runs per pair, for a total of $5 \times 10 \times 2 = 100$ runs; and
(iii) \textbf{global validation}: whenever a local update is predicted to
improve the global metric, we evaluate the new configuration on a held-out
validation set. We use 20 validation inputs and observe 7 such updates, but the
\emph{effective} cost is discounted by a factor of $2/3$ due to cached
trajectories:
\[
20 \times 7 \times \tfrac{2}{3} \approx 93.
\]
The discount factor arises because, for each validation, we do not always need
to downstream sample the entire system: roughly speaking, if $A$ is updated, we need 1 system run per validation input; if $B$ is updated, we need $2/3$ system run per validation input; if $C$ is updated, we need $1/3$ system run per validation input
\[
\frac{1 + \tfrac{2}{3} + \tfrac{1}{3}}{3} = \tfrac{2}{3}
\]
of a full system run. The resulting global performance is $24.24\%$.

For DSPy and TextGrad, we follow their standard setup and optimize only the
prompt for the final component $C$, while holding $A$ and $B$ fixed. To obtain
a fair cost comparison in terms of effective full system runs, we (i)
pre-compute a pool of 20 validation inputs by running $A$ and $B$ once, and
then (ii) run only $C$ during optimization. Assuming $A$, $B$, and $C$ have
comparable cost, we convert the number of $C$-only calls into an effective
number of full system runs by dividing by three. With 48 optimization steps for
TextGrad and 36 for DSPy, this yields
\[
20 \times 48 / 3 = 320 \quad \text{and} \quad 20 \times 36 / 3 = 240
\]
effective full system runs, corresponding to 0.32k and 0.24k in
Table~\ref{tab:amazon_cost}. These are conservative lower bounds, since any
additional evaluation of upstream components would only increase their cost.

Under this cost-normalized view, \optimas{} achieves the best global performance
while using a comparable (slightly lower) effective number of full system runs
than TextGrad and only moderately more than DSPy. This demonstrates that local
optimization with learned reward functions can deliver stronger performance
without incurring prohibitive cost.

\paragraph{Additional PPO training cost.}
Beyond system-run cost, \optimas{} incurs additional compute for training local
models with PPO. We measure this cost in GPU-hours. For each local PPO update
on the \amazon{} system, we train for 3 epochs on a single NVIDIA
A100-SXM4-80GB GPU. Averaged over 5 runs of local optimization, each run takes
approximately 12 minutes, yielding a total of about 6 GPU-hours (equivalently,
1.5 hours on 4 GPUs). This one-time training overhead is modest relative to the
cost of repeated system evaluations and is only required for systems with
trainable local components.